\documentclass[opre,nonblindrev]{informs3}
\SingleSpacedXI

\renewcommand{\theARTICLETOP}{}
\renewcommand{\theARTICLERULE}{\vspace{12pt}}
\renewcommand{\theRRHFirstLine}{}
\renewcommand{\theRRHSecondLine}{}
\renewcommand{\theLRHFirstLine}{}
\renewcommand{\theLRHSecondLine}{}
\renewcommand{\theECLRHFirstLine}{}%
\renewcommand{\theECLRHSecondLine}{}%
\renewcommand{\theECRRHFirstLine}{}%
\renewcommand{\theECRRHSecondLine}{}%
\renewcommand{\endproof}{\endTrivlist\addvspace{2ex}}
\makeatletter
\renewcommand\section{\@startsection {section}{1}{\z@}{-13pt plus -6pt minus -3pt}{4pt}%
  {\fs.13.15.\bfseries\RAGG}}%
\renewcommand\subsection{\@startsection{subsection}{2}{\z@}{-13pt plus -6pt minus -3pt}{4pt}%
  {\TEN\bfseries\RAGG}}%
\makeatother
\let\FigureCaptionFontStyle\undefined
\def\FigureCaptionFontStyle{\mdseries}

\usepackage[scaled=.98,sups]{XCharter}%

\usepackage[numbers]{natbib}
 \def\bibfont{\footnotesize}%
\TheoremsNumberedThrough   
\ECRepeatTheorems
\theoremstyle{TH}%
\OneAndAHalfSpacedXII
\usepackage{anyfontsize}
\usepackage{wrapfig}
\usepackage{xcolor}

\usepackage{setspace}

\DeclareMathAlphabet{\mathsfit}{T1}{\sfdefault}{\mddefault}{\sldefault}
\SetMathAlphabet{\mathsfit}{bold}{T1}{\sfdefault}{\bfdefault}{\sldefault}
\DeclareMathAlphabet{\mathcal}{OMS}{cmsy}{m}{n}
\usepackage[table, svgnames, dvipsnames]{xcolor}
\usepackage[hypertexnames=false]{hyperref}
\hypersetup{%
  colorlinks=true,%
  linkcolor={blue!66!black},%
  linkbordercolor=red,%
  citecolor={blue!66!black},
}
\usepackage{cite}
\usepackage{algorithm}
\usepackage{algorithmic}
\usepackage{mathtools}
\usepackage{booktabs,multirow}
\usepackage{enumitem}
\usepackage{array}
\usepackage{booktabs}

\newcolumntype{P}[1]{>{\raggedright\arraybackslash}p{#1}}
\setlist[enumerate,1]{label=\normalfont{(\Roman*)},leftmargin=*}

\newcommand*{\QED}{%
\leavevmode\unskip\penalty9999 \hbox{}\nobreak\hfill
    \quad\hbox{$\square$}%
}

\newcommand{\cR}{\mathcal{R}}

\newcommand{\diff}{\,\mathrm{d}}

\renewcommand{\emph}[1]{\textit{#1}}

\DeclarePairedDelimiterX{\inp}[2]{\langle}{\rangle}{#1, #2}
\definecolor{longhorn}{rgb}{0.8, 0.33, 0.0}

\TITLE{Variable Selection for Feature-Based Newsvendor}%
\ARTICLEAUTHORS{
\AUTHOR{Zhaoliang Yuan}
\AFF{School of Artificial Intelligence, The Chinese University of Hong Kong, Shenzhen, \EMAIL{225080014@link.cuhk.edu.cn}}
\AUTHOR{Jie Wang
}
\AFF{School of Artificial Intelligence and School of Data Science, The Chinese University of Hong Kong, Shenzhen, \EMAIL{jwang@cuhk.edu.cn}}
} %

\ABSTRACT{
Feature-based newsvendor models use observable covariates to tailor inventory decisions, aiming to balance holding and shortage costs under demand uncertainty. However, high-dimensional feature sets often hinder interpretability and inflate data collection and implementation costs. This paper studies variable selection for the feature-based newsvendor problem under a hard cardinality constraint on the number of selected features. We formulate the resulting \(\ell_0\)-constrained empirical newsvendor problem with \(\ell_2\)-regularization, establish its computational hardness, and develop a mixed-integer second-order cone programming reformulation that strengthens the standard Big-\(M\) formulation. To enable scalability beyond exact optimization, we develop a randomized-rounding algorithm with a bi-criteria guarantee and a greedy heuristic. Statistically, we provide theoretical analysis of the resulting sparse policy estimator, including finite-sample estimation error, out-of-sample risk bounds, and support recovery guarantees. Extensive experiments on both synthetic and real data illustrate the computational and statistical trade-offs among various baselines. Our results demonstrate that the proposed variable selection framework achieves competitive out-of-sample operational costs while using substantially fewer covariates.
}

\begin{document}

\maketitle

\section{Introduction}
The newsvendor problem is a classic inventory model that seeks to determine the optimal order quantity by balancing the cost of overstocking and understocking in the face of random demand. It is crucial to obtain an accurate estimate of customer demand, which is influenced by a multitude of features, including weather, time, location, and economic conditions. In the big data era, leveraging available data that pairs these influential features (covariates) with corresponding demand observations enables us to make data-driven predictions of future customer demand and more reliable decisions, while ignoring the presence of covariates can lead to inconsistent decisions~\citep{ban2019big}.

A natural paradigm for incorporating features into decision-making is contextual stochastic optimization (CSO)~\citep{sadana2025survey}. Broadly speaking, CSO can be divided into three classes: decision-rule optimization~\citep{liyanage2005practical, ban2019big}, estimate-then-optimize~\citep{bertsimas2020predictive}, and integrated estimation and optimization~\citep{donti2017task, kallus2023stochastic, qi2025integrated}. In this paper, we focus on the decision-rule optimization approach, aiming to learn an end-to-end mapping from features to decisions. Linear decision rules are attractive because of their simplicity and tractability~\citep{ban2019big}, but they may not capture complex feature-to-decision relationships. To improve out-of-sample performance , nonlinear decision rules based on kernel methods \citep{bertsimas2022data, notz2022prescriptive} and deep learning \citep{oroojlooyjadid2020applying, zhang2017assessing, huber2019data} have also been studied.

Despite this progress, the interpretability of learned decision rules has received comparatively less attention. In many applications, only a small subset of features is truly decision-relevant. Identifying such
features can improve interpretability, reduce data collection and monitoring costs, simplify downstream implementation, and improve robustness when irrelevant features are noisy or unstable. This motivates us to consider variable selection for feature-based newsvendor models.

Variable selection has been extensively studied in statistics and machine learning through sparsity-inducing formulations, including Lasso~\citep{tibshirani1996regression}, bridge-type penalties~\citep{huang2008asymptotic}, the smoothly clipped absolute deviation (SCAD) penalty~\citep{xie2009scad,Fan01122001}, and the minimax concave penalty (MCP)~\citep{zhang2010nearly}. These methods can be viewed as tractable surrogates for the ideal but computationally challenging \(\ell_0\)-constrained formulation. Convex penalties are typically easier to optimize, whereas nonconvex penalties may achieve sharper statistical properties but are harder to solve globally. In this paper, we study the following question:
\begin{quote}
\centering{\bf How can we efficiently solve variable selection for feature-based newsvendor and what statistical guarantees can be established?
}
\end{quote}
Our contributions are summarized as follows.
\begin{enumerate}
    \item 
\textbf{Exact and approximation algorithms.}
For the variable selection problem, we add an $\ell_2$-norm regularization and reformulate it as a mixed-integer second-order cone program~(MISOCP).
Compared with the standard Big-$M$ approach for variable selection, our MISOCP formulation yields a tighter continuous relaxation and enables a more efficient algorithm design.
Based on the reformulation, we further develop two efficient approximation algorithms and show that the  randomized-rounding method admits a bi-criteria guarantee. 
    \item \textbf{Statistical guarantees.}
Under the assumption that demand has sparse and linear dependence on features, we establish the concentration error bound between our estimated policy and the ground truth optimal policy, as well as the out-of-sample performance and support recovery guarantees.
These performance bounds break the curse of dimensionality because they crucially depend on the sparsity level instead of the data dimension. %
    \item \textbf{Numerical validation.}
We compare the exact optimization methods, randomized-rounding, greedy, and Lasso-based approaches using synthetic and retail data, illustrating their computational and statistical trade-offs.
\end{enumerate}

\subsection{Related Works}
\noindent\textbf{Data-Driven Newsvendor.}
The classical newsvendor model assumes that the demand distribution is known, whereas modern data-driven approaches learn ordering policies directly from historical data. Early work along these lines includes operational-statistics-based approaches \citep{liyanage2005practical} and feature-based linear policies for the big-data newsvendor \citep{ban2019big}. Subsequent work has developed richer data-driven models based on quantile regression and machine learning \citep{huber2019data}, reproducing kernel Hilbert space (RKHS)-based policy classes \citep{bertsimas2022data, bertsimas2020predictive}, and deep neural networks \citep{oroojlooyjadid2020applying, zhang2017assessing}. 
 \citet{serrano2024bilevel} provided a variable selection framework for the feature-based newsvendor problem. 
 They select key features by adding the sparsity budget constraint to the objective and tune penalty parameters using bilevel optimization. Our work contributes to this literature by studying hard-cardinality-constrained policy learning and by providing both computational and statistical guarantees.
\citet{feature_rich} also studied this framework, but their computational approach can be viewed as a heuristic for solving our $\ell_0$-constrained formulation, and their statistical analysis focused on the data privacy viewpoint.

\noindent\textbf{Quantile Regression.}
A fundamental observation in the newsvendor problem is that the optimal order quantity is a conditional quantile of the demand distribution. This connection links contextual newsvendor problems closely to quantile regression, beginning with the seminal work of \citet{koenker1978regression}. Quantile-regression-based methods have been used successfully in feature-based newsvendor models \citep{huber2019data}. In high-dimensional settings, sparse quantile regression has been studied under both \(\ell_1\)-regularized and \(\ell_0\)-type formulations \citep{belloni2011l1, dai2023variable}. The study in this paper can be interpreted as an \(\ell_0\)-constrained and \(\ell_2\)-regularized empirical quantile regression problem.

\noindent\textbf{Variable Selection.}
Variable selection is a central topic in statistics and machine learning. Classical approaches promote sparsity through penalized formulations~\citep{tibshirani1996regression}, or through greedy and local-search methods~\citep{li2024sparse, Exact2025, xie2020scalable, li2024beyond, civril2009selecting} that approximately solve $\ell_0$-constrained optimization problems. Although $\ell_0$-constrained formulations are generally NP-hard and computationally challenging in the worst case, they often perform well in high-dimensional, small-sample, and low-signal-to-noise settings~\citep{Bertsimas_2021, AtamturkRankoneCF, xie2020scalable}. This motivates the development of scalable algorithms for directly solving the $\ell_0$-constrained variable-selection problem for the newsvendor model, as well as theoretical analyses that quantify the optimality gap of approximation algorithms. Compared with classical sparse ridge regression, our problem is further complicated by an empirical loss that is piecewise linear and asymmetric. To address these challenges, we combine perspective reformulations and support-selection algorithms with a statistical analysis tailored to the quantile-regression structure of the newsvendor loss.

\noindent\textbf{Notations. }
For real numbers $a,b$, we let $a\lor b=\max\{a,b\}$.
Given a positive integer $n$, define $[n]=\{1,\ldots,n\}$.
For a distribution function $F(\cdot)$, we denote $F^{-1}(p)=\inf\{x\in\mathbb{R}:~F(x)\ge p\}, p\in[0,1]$.
Given an event $E$, define $\textbf{1}_{E}(x)=1$ if $x\in E$ and otherwise $\textbf{1}_{E}=0$.
For a vector $s\in\mathbb{R}^n$, we use $s[i]$ to denote its $i$-th component, and define its $\ell_0$-norm $\|s\|_0=\sum_{i\in[n]}\textbf{1}_{s[i]\ne0}$.
Given an $m\times n$ matrix $\textbf{X}$ and a set $S\subseteq[n]$, let $\textbf{X}_S$ denote the submatrix of $\textbf{X}$ with columns from the set $S$.
Given a set $\mathcal{F}$, let $\mathrm{Conv}(\mathcal{F})$ denote its convex hull, i.e., the smallest convex set containing $\mathcal{F}$.

\section{Problem Setup}
The newsvendor problem seeks to solve the following stochastic optimization problem:
\begin{equation}\label{Eq:newsvendor}
\min_{z\geq 0}~\mathbb{E}[b\max\{y-z,0\} + h\max\{z-y,0\}],
\end{equation}
where the expectation is taken with respect to the random demand $y$, $z$ denotes the order quantity, and $b$, $h$ represent the unit back-ordering and holding costs, respectively.
A central challenge in solving \eqref{Eq:newsvendor} is that the underlying demand distribution is typically unknown.

In the context of big data, auxiliary features can be used for accurate demand estimation.
Let \(x\in\mathbb{R}^D\) denote a feature vector. For a given \(x\), the optimal order quantity can be expressed as a function $f(x)$ with
\[
f(x) = \underset{z\ge0}{\mathrm{argmin}}~\Big\{\mathbb{E}_{y\mid x}\big[b\max\{y-z,0\} + h\max\{z-y,0\}\big]\Big\}.
\]
Rather than first estimating the conditional demand distribution and then solving the resulting newsvendor problem, one can directly learn an end-to-end policy function \(f\) that maps features to decisions by solving the following contextual stochastic optimization~(CSO) problem:
\begin{equation}
\min_{f\in\mathcal{F}}~\mathbb{E}_{(x,y)}\big[b\max\{y-f(x),0\} + h\max\{f(x)-y,0\}\big].\label{Eq:CSO}
\end{equation}
The choice of the function class $\mathcal{F}$ critically influences both computational tractability and out-of-sample performance. In this paper, we focus on sparse linear decision rules of the form
\[
\mathcal{F} = \Big\{
f(x) = \omega^\top x+c:~\omega\in\mathbb{R}^D, c\in\mathbb{R}, \|\omega\|_0\le d
\Big\}.
\]
The \(\ell_0\)-constraint ensures that at most \(d\) covariates actively affect the decision, thereby facilitating interpretable and explainable policies. 
Under the linear decision-rule model, the subproblem associated with each fixed support is a convex optimization and therefore enables us to develop mixed-integer convex program reformulations for our variable selection framework. This framework can be extended to nonlinear decision rules through basis expansion.
\begin{remark}[Linear Demand Model]
Consider the linear demand model $y = (\omega^*)^\top x + \epsilon$, where $\epsilon$ is a random variable independent of $x$ with cumulative distribution function $F_\epsilon$. 
Let $\tau = b/(b+h)$.
The optimal feature-dependent newsvendor decision is the conditional \(\tau\)-quantile of \(y\) given \(x\), namely $f^*(x)=(\omega^*)^\top x+F_\varepsilon^{-1}(\tau).$
Under this linear demand model, the slope in the linear decision rule recovers the true coefficient \(\omega^*\), while the intercept estimates the \(\tau\)-quantile of the noise.
\end{remark}

\begin{remark}[Nonnegativity of order quantities]
The optimal newsvendor decision has the nonnegative constraint.
For the linear decision rule approach, we clip the estimated optimal order quantity to zero if it is negative. When demand is always non-negative, this procedure does not increase the newsvendor loss. 
In the main analysis, we focus on the unconstrained formulation in \eqref{Eq:CSO} for simplicity.
\end{remark}

The objective in \eqref{Eq:CSO} involves the population expectation under the joint distribution of \((x,y)\). In practice, however, only independent and identically distributed~(i.i.d.) samples \(\{(x_i,y_i)\}_{i=1}^n\) are available. We therefore replace the population expectation with the empirical average and include an \(\ell_2\)-regularization term to improve both computational tractability and out-of-sample performance, which will be discussed in later sections.
In summary, it leads to the optimization problem
\begin{equation}\label{eq:5}
\min_{\omega,c:~\|\omega\|_0 \leq d} ~~ F(\omega,c):=\frac{1}{n}\sum_{i=1}^n\Big[b \cdot \max\{0, y_i - \omega^\top x_i-c\} + h \cdot \max\{0, \omega^\top x_i+c - y_i\}\Big] + \frac{1}{2\gamma}\|\omega\|^2_2,
\end{equation}
Problem~\eqref{eq:5} is nonconvex because of the \(\ell_0\)-constraint, and is therefore generally difficult to solve. In fact, we show that it is NP-hard by a reduction from the exact cover by 3-sets problem.
Its proof is provided in Appendix~\ref{Proof:Proposition:hardness}.
\begin{proposition}[Computational Hardness]\label{prop:hardness}
For fixed \(\gamma,b,h>0\), Problem~\eqref{eq:5} is NP-hard whenever the sparsity level \(d\) is part of the input.
\end{proposition}

Problem~\eqref{eq:5} is also related to sparse quantile linear regression. Indeed, it is equivalent to
\[
\min_{\omega,c:\ \|\omega\|_0\le d}
\left\{
\frac{(b+h)}{n}\sum_{i=1}^n
\rho_{\tau}\big(y_i-(\omega^\top x_i+c)\big)
+
\frac{1}{2\gamma}\|\omega\|_2^2
\right\},
\]
where $\tau=\frac{b}{b+h}$ and $\rho_{\tau}(u)=u\big(\tau-\mathbf{1}_{\{u<0\}}\big)$. 
In other words, the goal of variable selection for the feature-based newsvendor is equivalent to estimating sparse coefficients that minimize the \(\ell_2\)-regularized empirical quantile regression loss.

The remainder of the paper is organized as follows. In Section~\ref{sec:exact}, we present a mixed-integer second-order cone reformulation of Problem~\eqref{eq:5}, which can be solved exactly for medium-sized instances. In Section~\ref{sec:approximation}, we develop two approximation algorithms with performance guarantees. In Section~\ref{Sec:statistical}, we establish statistical guarantees for the proposed approach, including consistency, sample complexity, and support recovery.

\section{Exact Computational Algorithms}\label{sec:exact}

In this section, we develop exact algorithms for solving Problem~\eqref{eq:5} to global optimality. A standard approach is the \emph{Big-$M$} reformulation proposed by \citet{dai2023variable}, which converts the sparsity-constrained problem into a mixed-integer convex program.

\noindent\textbf{Big-$M$ Method.}
Recall that the constraint \(\|\omega\|_0 \le d\) requires that at most \(d\) entries of \(\omega\) be nonzero. To model the support of \(\omega\), we introduce a binary vector \(s\in\{0,1\}^D\), where \(s[i]=1\) indicates that \(\omega[i]\) is allowed to be nonzero. Equivalently, we impose
$\mathbf{1}\{\omega[i]\neq 0\}\le s[i], i\in[D].$
Assume there exists a vector \(M\in\mathbb{R}^D_+\) such that
$|\omega[i]|\le M[i],  i\in[D],$ for some optimal solution \(\omega\). Then the logical constraint can be equivalently enforced by the linear inequality $|\omega[i]|\le M[i]\,s[i],  i\in[D].$ This observation leads to the following reformulation of Problem~\eqref{eq:5}.
\begin{proposition}[Big-$M$ reformulation~{\citep{dai2023variable}}]\label{prop:big:M}
Assume there exists a vector \(M\in\mathbb{R}^D_+\) such that an optimal solution \(\omega^*\) to Problem~\eqref{eq:5} satisfies $|\omega^*[i]|\le M[i],~\forall i\in[D].$
Then Problem~\eqref{eq:5} is equivalent to
\begin{equation}\label{eq:bigM}
\begin{aligned}
\min_{\omega,c,s}\quad
& \frac{1}{n}\sum_{i=1}^n\Big[
b\max\{0,\, y_i-\omega^\top x_i-c\}
+
h\max\{0,\, \omega^\top x_i+c-y_i\}
\Big]
+\frac{1}{2\gamma}\|\omega\|_2^2\\
\text{s.t.}\quad
& \sum_{i\in[D]} s[i]\le d,~~~s\in\{0,1\}^D,\\
& |\omega[i]|\le M[i]\,s[i],\qquad \forall i\in[D].
\end{aligned}
\tag{Big-$M$}
\end{equation}
\end{proposition}
Big-$M$ method is standard but often computationally expensive for large-scale mixed-integer problems. In practice, such formulations are typically solved by branch-and-bound, which repeatedly solves convex relaxations. A major difficulty is that the Big-$M$ formulation usually yields a weak relaxation: the optimality gap between \eqref{eq:bigM} and its continuous relaxation, obtained by replacing \(s\in\{0,1\}^D\) with \(s\in[0,1]^D\), can be substantial. 

Some heuristics can help derive tight bounds of $M$ and improve numerical performance. For example, let $(\omega_0,c_0)$ be a near-optimal solution to Problem~\eqref{eq:5}, which can be obtained using some tractable algorithms outlined in Section~\ref{sec:approximation}.
Any optimal solution to Problem~\eqref{eq:5} satisfies that 
\[
\frac{\|\omega^*\|_2^2}{2\gamma}\le F(\omega^*,c^*)\le F(\omega_0,c_0).
\]
Based on this relation, the vector $M$ with $M[i]=\sqrt{2\gamma F(\omega_0,c_0)}, i\in[D]$ is a valid choice because $|\omega^*[i]|\le\|\omega^*\|_2\le M[i]$.
However, 
solving the Big-$M$ formulation remains challenging for large-scale datasets.

Motivated by this limitation, we next introduce an alternative reformulation that is more computationally tractable. We show that it has a tighter convex relaxation than the Big-$M$ approach and is therefore better suited for exact computation.

\noindent\textbf{Perspective Reformulation.}
We introduce a stronger reformulation than the Big-$M$ approach, based on the perspective function. 
To this end, we first rewrite the quadratic regularization term by introducing auxiliary variables \(\mu\in\mathbb{R}^D_+\):
\[
\frac{1}{2\gamma}\|\omega\|_2^2
=
\min_{\mu\in\mathbb{R}_+^D}
\left\{
\frac{\|\mu\|_1}{2\gamma}:
\omega[i]^2\le \mu[i],\ \forall i\in[D]
\right\}.
\]
Therefore, Problem~\eqref{eq:5} is equivalent to the following extended Big-$M$ formulation:
\begin{subequations}\label{eq:bigM:extended}
\begin{align}
\min_{\omega,c,s,\mu}\quad
& \frac{1}{n}\sum_{i=1}^n\Big[
b \cdot \max\{0, y_i-\omega^\top x_i-c\}
+
h \cdot \max\{0, \omega^\top x_i+c-y_i\}
\Big]
+\frac{\|\mu\|_1}{2\gamma},\\
\text{s.t.}\quad
& \sum_{j\in[D]} s[j]\le d, \quad s\in\{0,1\}^D,\\
& (\omega[j], s[j], \mu[j])\in\mathcal{S}_j,\quad \forall j\in[D],\label{eq:bigM:extended:c}
\end{align}
\end{subequations}
where for each coordinate $j\in[D]$, define the set 
\[
\mathcal{S}_j
:=
\Big\{
(\omega[j], s[j], \mu[j])\in\mathbb{R}\times\{0,1\}\times\mathbb{R}_+:
\omega[j]^2\le \mu[j],\ 
|\omega[j]|\le M[j]\,s[j]
\Big\}.
\]
\citet{gunluk2011perspective} have characterized the convex hull of $\mathcal{S}_j$ as
\[
\operatorname{Conv}(\mathcal{S}_j)
=
\Big\{
(\omega[j], s[j], \mu[j])\in\mathbb{R}\times[0,1]\times\mathbb{R}_+:
\omega[j]^2\le \mu[j] s[j],\ 
|\omega[j]|\le M[j]\,s[j]
\Big\}.
\]
By convexifying the constraint \eqref{eq:bigM:extended:c}, we can obtain the perspective reformulation of \eqref{eq:5}.
\begin{proposition}[Perspective Reformulation]
Under the same condition as in Proposition~\ref{prop:big:M}, 
Problem~\eqref{eq:5} is equivalent to 
\begin{equation}\label{eq:misocp:final}\tag{Perspective}
\begin{aligned}
\min_{\omega,c,s,\mu}\quad
& \frac{1}{n}\sum_{i=1}^n\Big[
b \cdot \max\{0, y_i-\omega^\top x_i-c\}
+
h \cdot \max\{0, \omega^\top x_i+c-y_i\}
\Big]
+\frac{1}{2\gamma}\sum_{j\in[D]}\mu[j],\\
\text{s.t.}\quad
& \sum_{j\in[D]} s[j]\le d, \quad s\in\{0,1\}^D,\\
& (\omega[j], s[j], \mu[j])\in\operatorname{Conv}(\mathcal{S}_j),\quad \forall j\in[D].
\end{aligned}
\end{equation}
\end{proposition}
Since the constraint \(\omega[j]^2\le \mu[j]\,s[j]\) is representable by a rotated second-order cone when \(s[j]\ge 0\) and \(\mu[j]\ge 0\), Problem~\eqref{eq:misocp:final} is a mixed-integer second-order cone program~(MISOCP).
The following theorem compares the Big-$M$ and perspective formulations.
Its proof follows techniques in \citep{xie2020scalable} and is provided in Appendix~\ref{Appendix:thm:perspective-vs-bigM}.

\begin{theorem}[Strength of Perspective Reformulation]\label{thm:perspective-vs-bigM}
\begin{enumerate}
    \item 
The mixed-integer formulations \eqref{eq:bigM:extended} and \eqref{eq:misocp:final} are equivalent, and hence have the same optimal value;
    \item 
Let $V_{\mathrm{rel}}^{\text{Big-M}}$ and $V_{\mathrm{rel}}^{\text{Pers}}$ denote the optimal values of the continuous relaxations of \eqref{eq:bigM} and \eqref{eq:misocp:final}, respectively.
It always holds that $V_{\mathrm{rel}}^{\text{Big-M}} \le V_{\mathrm{rel}}^{\text{Pers}}$.
\end{enumerate}
\end{theorem}

\begin{remark}[Perspective Reformulation is Strictly Stronger]\label{Remark:dominance}
We provide an example to show that the perspective reformulation is strictly stronger.
Take $D=2, d=1, b=h=1, n=4, M=2, \gamma=1$. Assume that one has access to four data samples: $\{(e_1,1), (-e_1,-1), (e_2,1), (-e_2,-1)\}$, where $e_1,e_2$ denote the unit vectors in $\mathbb{R}^2$.
In this case, it is easy to verify $V_{\mathrm{rel}}^{\text{Big-M}}=0.75$ and $V_{\mathrm{rel}}^{\text{Pers}}=0.875$.
Hence, this problem instance ensures that $V_{\mathrm{rel}}^{\text{Big-M}} < V_{\mathrm{rel}}^{\text{Pers}}$.
\end{remark}

In summary, perspective reformulation is preferred to the Big-$M$ method in two ways.
First, Theorem~\ref{thm:perspective-vs-bigM} together with Remark~\ref{Remark:dominance} explains that the perspective reformulation is more computationally efficient. Since its continuous relaxation is tighter than that of the Big-$M$ formulation, a branch-and-bound solver typically explores fewer nodes and converges faster in practice. 
Second, it is difficult to derive valid and tight Big-$M$ constants for solving both formulations, and using overly conservative Big-$M$ constants makes the formulations numerically ill-conditioned to solve.
Even if we drop the constraints $|\omega[j]|\le M[j]s[j], j\in[D]$ in Problem~\eqref{eq:misocp:final}, it is still a valid reformulation of Problem~\eqref{eq:5} and computationally tractable, leading to the following \emph{Big-M-free} reformulation:
\begin{equation}\label{Eq:big:M:free:pers}
\begin{aligned}
\min_{\omega,c,s,\mu}\quad
& \frac{1}{n}\sum_{i=1}^n\Big[
b\max\{0,\, y_i-\omega^\top x_i-c\}
+
h\max\{0,\, \omega^\top x_i+c-y_i\}
\Big]
+\frac{1}{2\gamma}\sum_{j\in[D]}\mu[j]\\
\mbox{s.t.}&\quad 
\sum_{j\in[D]}s[j]\le d,\quad s\in\{0,1\}^D,\\
&\quad \omega[i]^2\le \mu[i] s[i], i\in[D],\quad \mu\in\mathbb{R}_+^D.
\end{aligned}
\tag{Big-$M$-free-Perspective}
\end{equation}

\section{Scalable Approximation Algorithms}\label{sec:approximation}

In this section, we study two scalable approximation algorithms to obtain near-optimal solutions to Problem~\eqref{eq:5} and build their performance guarantees.

\subsection{Continuous Relaxation with Randomized Rounding}
Our first approach is motivated by the MISOCP~\eqref{Eq:big:M:free:pers}.
Its continuous relaxation replacing the binary constraint $s\in\{0,1\}^D$ with $s\in[0,1]^D$ becomes a convex program, and can be efficiently solved using off-the-shelf solvers.
Then, we estimate the support of sparse regression coefficients using a randomized rounding scheme.
See Algorithm~\ref{Alg:con:algorithm} for detailed implementation.
\begin{algorithm}[!ht]
	\renewcommand{\algorithmicrequire}{\textbf{Input:}}
    \renewcommand{\algorithmicensure}{\textbf{Output:}}
    \caption{Continuous Relaxation with Randomized Rounding for Problem \eqref{eq:5}}
	\begin{algorithmic}[1]\label{Alg:con:algorithm}
    \REQUIRE Data \(\{(x_i,y_i)\}_{i=1}^n\), sparsity budget \(d\), regularization parameter \(\gamma\)
    \STATE Solve the continuous relaxation of~\eqref{Eq:big:M:free:pers} and obtain an optimal solution
    \((\widehat{\omega},\widehat c,\widehat s,\widehat\mu)\)
    \STATE Initialize set $S=\emptyset$%
    \STATE \textbf{for} $i=1,\ldots,D$~\textbf{do}
    \STATE \qquad Sample an uniform number $U$ on $[0,1]$.
    \STATE \qquad Update $S=S\cup\{i\}$ if $U\le \widehat{s}[i]$. 
    \STATE \textbf{end for}
    \STATE Obtain $(\tilde{\omega}, \tilde{c})$ as the optimal solution to $    \underset{\omega, c}{\arg\min}~\Big\{ 
F(\omega,c):~\omega[i]=0\text{ if }i\notin S
\Big\}.$
    \RETURN $(\tilde{\omega}, \tilde{c})$ as the regression coefficient estimator
    \end{algorithmic}
\end{algorithm}

We provide performance guarantees for this computationally efficient algorithm below.
Let $\textbf{X}=(x_1,\ldots,x_n)^\top\in\mathbb{R}^{n\times D}$ be the data matrix. For $k\in[D]$, define
\[
\theta_k := \max_{|S|=k}~\sigma_{\max}(\textbf{X}_S\textbf{X}_S^\top),
\]
where $\sigma_{\max}(\cdot)$ denotes the maximum singular value.
Due to the randomized rounding procedure, the support size of the estimated regression coefficient may not satisfy the sparsity budget.
The following theorem provides the bi-criteria approximation guarantee of Algorithm~\ref{Alg:con:algorithm}. It proves that with high probability, the selected support is only moderately larger than \(d\), while the
objective value is within an additive tolerance of the optimal value.
\begin{theorem}[Approximation Guarantee of Algorithm~{\ref{Alg:con:algorithm}}]\label{Theorem:approximation:guarantee}
Given $\delta\in(0,1)$ and $\epsilon>0$. 
Let $(\tilde{\omega},\tilde{c})$ be the output from Algorithm~\ref{Alg:con:algorithm}.
Suppose the regularization coefficient
\[
\frac{1}{2\gamma}\ge
\frac{\log(2n/\delta)(h\lor b)^2{\theta_1}}{6n\epsilon} + \frac{(h\lor b)^2\sqrt{2\theta_1\theta_d\log(2n/\delta)}}{4n\epsilon}.
\]
Then, with probability at least $1-\delta$, it holds that
\[
\|\tilde{\omega}\|_0\le \left( 
1 + \sqrt{\frac{2\log(2/\delta)}{d}} + \frac{\log(2/\delta)}{d}
\right)d,\quad
F(\tilde{\omega}, \tilde{c})\le \text{optval}\eqref{eq:5} + \epsilon.
\]

\end{theorem}
The proof of Theorem~\ref{Theorem:approximation:guarantee} is provided in Appendix~\ref{proof:thm:app}.
The proof step to bound the sparsity level of $\tilde{\omega}$ follows from the usage of concentration inequality discussed in \citep[Theorem~5]{xie2020scalable}. The reference therein implicitly assumed that $\log(2/\delta)\le d/3$. Our concentration result refines the concentration result and does not impose this assumption.
The support-size bound in Theorem~\ref{Theorem:approximation:guarantee} depends on the target sparsity \(d\), but not on the data dimension $D$. The objective performance guarantee requires that the regularization parameter $\frac{1}{2\gamma}$ is not too small: its lower bound is related to the dimension $D$ through the maximum singular values $\theta_1$ and $\theta_d$ from the design matrix.
This assumption is not quite restrictive since adding $\ell_2$-norm regularization ensures the estimated solution is more robust than unregularized models~\citep{blanchet2019robust}.

\subsection{Greedy Algorithm}

Greedy algorithm is a common heuristic for many variable selection problems. To develop a greedy algorithm for solving \eqref{eq:5}, we rewrite it as a combinatorial formulation. Indeed, by introducing a subset $S\subseteq[D]$, Problem~\eqref{eq:5} is equivalent to
\begin{equation}\label{eq:6}
    \begin{aligned}
        \min_{S\subseteq[D], |S|\le d}~\left\{\min_{\substack{\omega\in \mathbb{R}^D, c\in\mathbb{R}}}
        F(\omega, c):~\omega[i]=0\text{ if }i\notin S
        \right\}.
    \end{aligned}
\end{equation}
For a fixed support \(S\), the inner subproblem of \eqref{eq:6} is a convex program. Using strong duality theory on the inner minimization problem, this subproblem has an equivalent dual representation that depends on \(S\) only through submatrix \(\mathbf X_S\).
\begin{theorem}[Minimax Reformulation of \eqref{eq:6}]\label{Theorem:MIQP reformulation}
Problem \eqref{eq:6} has the same optimal value as the following minimax optimization problem:
\begin{equation}\label{Eq:greedy}
    \min_{S\subseteq[D], |S|\le d} ~~\left\{G(S):=\max_{z[i]\in[-b,h], i\in[n],\sum_{i=1}^n z[i]=0} -z^\top
    \left( 
\frac{\gamma}{2n^2}\textbf{X}_S\textbf{X}_S^\top
    \right)z-\frac{1}{n}z^\top y\right\}.
\end{equation}
\end{theorem}

Based on Theorem~\ref{Theorem:MIQP reformulation}, Greedy algorithm proceeds by iteratively finding a candidate index such that the objective $G$ is minimized for current subset including the new index.
See the detailed procedure in Algorithm~\ref{Alg:Eq:greedy}.
\begin{algorithm}[H]
	\renewcommand{\algorithmicrequire}{\textbf{Input:}}
    \renewcommand{\algorithmicensure}{\textbf{Output:}}
    \caption{Greedy Algorithm for Problem \eqref{Eq:greedy}}
    \label{alg3}
	\begin{algorithmic}[1]\label{Alg:Eq:greedy}
    \REQUIRE Data \(\{(x_i,y_i)\}_{i=1}^n\), sparsity budget \(d\), regularization parameter $\gamma$
    \STATE Initialize \(S_0=\emptyset\)
\FOR{\(k=0,1,\ldots,d-1\)}
    \STATE Choose $j_k\in
    \arg\min_{j\in[D]\setminus S_k}
    G(S_k\cup\{j\})$
    \STATE Update \(S_{k+1}=S_k\cup\{j_k\}\)
\ENDFOR
\STATE Obtain $(\tilde{\omega}^{\mathrm{gr}}, \tilde{c}^{\mathrm{gr}})$ as the optimal solution to $    \underset{\omega, c}{\arg\min}~\Big\{ 
F(\omega,c):~\omega[i]=0\text{ if }i\notin S_d
\Big\}.$
\RETURN $(\tilde{\omega}^{\mathrm{gr}}, \tilde{c}^{\mathrm{gr}})$
    \end{algorithmic}
\end{algorithm}

By Theorem~\ref{Theorem:MIQP reformulation}, it can be shown that the set function $G(S)$ is non-increasing in $S$, and the objective of iterates in Algorithm~\ref{Alg:Eq:greedy} is monotonically decreasing. 
It is an open question on the approximation guarantees of the final iterate from the greedy algorithm, as the set function $G$ does not satisfy the submodular assumption commonly used in literature.
Inspired by \citep{xie2020scalable}, the greedy approach can also be integrated with Algorithm~\ref{Alg:con:algorithm}.
Specifically, given threshold $\delta>0$ and $\widehat{\omega}$ obtained from continuous relaxation, define $\mathcal{C}=\{i:~|\widehat{\omega}[i]|>\delta\}$ as the effective support obtained from continuous relaxation.
We can apply the greedy algorithm based on $\mathcal{C}$ instead of $[D]$ and can reduce the number of candidate variables. This algorithm works well especially when $\widehat{\omega}$ is sparse.

\section{Statistical Performance Guarantees}\label{Sec:statistical}

In this section, we investigate statistical performance guarantees for our proposed variable selection framework.
Assume the observed data samples \(\{(x_i,y_i)\}_{i=1}^n\) are i.i.d. copies of
    \((X,Y)\), where
\begin{equation}
    Y=(\omega^*)^\top X+c^*+\varepsilon.\label{Eq:model}
\end{equation}
Here, $\omega^*\in\mathbb{R}^D$ is a fixed $d$-sparse vector with \(S^*=\operatorname{supp}(\omega^*)\), and
    \(|S^*|=d\), and $c^*\in\mathbb{R}$ is a fixed intercept parameter.
    The distribution of the noise $\varepsilon$ is independent of $X$.
We analyze the estimator from the optimization problem
\[
(\widehat{\omega}_n, \widehat{c}_n)=
\underset{\omega,c:~\|\omega\|_0 \leq d}{\arg\min}
~~\left\{\frac{1}{n}\sum_{i=1}^n\Big[b  \max\{0, y_i - \omega^\top x_i-c\} + h \max\{0, \omega^\top x_i+c - y_i\}\Big] + \frac{\|\omega\|^2_2}{2\gamma_n}\right\},
\]
where the regularization parameter $\gamma_n$ is allowed to vary with the sample size $n$.
We impose the following technical assumptions throughout our analysis.

\begin{assumption}[Structure of Data]\label{Assumption:structure:data}
\begin{enumerate}
    \item\label{Assumption:structure:data:I}
The distribution of $\varepsilon$ satisfies $F_{\epsilon}(0):=\mathbb{P}(\epsilon\le 0)=\tau$, where $\tau=b/(b+h)$.
    \item\label{Assumption:structure:data:I:add}
The random variable $\varepsilon$ satisfies $\mathbb{E}|\varepsilon|<\infty$.
    \item\label{Assumption:structure:data:II}
The density of the noise exists and is denoted as $f_{\varepsilon}(\cdot)$.
There exists $r_0$ such that for all $|u|\le r_0$, it holds that $f_{\varepsilon}(u)\ge\underline{f}>0$.
    \item\label{Assumption:structure:data:III}
The augmented covariate
    \(\widetilde X=(X^\top,1)^\top\) is bounded almost surely: $\|\widetilde X\|_2\le K.$
    \item\label{Assumption:structure:data:IV}
There exists $\kappa>0$ such that for each $u\in\mathbb{R}^D, t\in\mathbb{R}, \|u\|_0\le 2d$, it holds that
\[
\mathbb{E}\big[ 
(u^\top X + t)^2
\big]\ge \kappa\big( 
\|u\|_2^2 + t^2
\big).
\]
\end{enumerate}
\end{assumption}
It is noteworthy that Assumption~\ref{Assumption:structure:data}\ref{Assumption:structure:data:I} is not restrictive. Indeed, \eqref{Eq:model} can be reformulated as $Y = (\omega^*)^\top X + \tilde{c}^* + \tilde{\varepsilon}^*$ for constant $\tilde{c}^*=c^* + F_{\varepsilon}^{-1}(\tau)$ and random noise $\tilde{\varepsilon}^*={\varepsilon} - F_{\varepsilon}^{-1}(\tau)$. Such a scalar transformation always ensures Assumption~\ref{Assumption:structure:data}\ref{Assumption:structure:data:I} holds.
Assumption~\ref{Assumption:structure:data}\ref{Assumption:structure:data:I:add} is a sufficient condition that ensures the optimal newsvendor risk
\[
\mathbb{E}_{(X,Y)}\big[ 
b\max(Y - (\omega^*)^\top X - c^*,0)
+
h\max((\omega^*)^\top X + c^* - Y, 0)
\big]=(b+h)\mathbb{E}_{\varepsilon}[\rho_{\tau}(\varepsilon)]
\]
is finite.
Assumption~\ref{Assumption:structure:data}\ref{Assumption:structure:data:II} imposes local regularity of the noise density around the target quantile. Assumption~\ref{Assumption:structure:data}\ref{Assumption:structure:data:III} imposes the boundedness of the feature vector. Assumption~\ref{Assumption:structure:data}\ref{Assumption:structure:data:IV} is a restricted eigenvalue condition, which ensures identifiability over sparse directions.
These assumptions are commonly used in sparse statistical analysis~\citep{wainwright2019high}. 

We first establish a finite-sample estimation-error bound regarding the estimator $(\widehat{\omega}_n, \widehat{c}_n)$ in Theorem~\ref{Theorem:error:bound}.
It reveals that under proper scaling of the ridge parameter, the estimator converges to the ground truth with a rate crucially dependent on the sparsity level instead of the data dimension. The estimation error vanishes quickly as the sample size $n$ increases.

\begin{theorem}[Estimation Error Bound]\label{Theorem:error:bound}
Let $\delta\in(0,1)$ and suppose Assumption~\ref{Assumption:structure:data} holds.
Set the ridge parameter 
$\frac{1}{2\gamma_n}\lesssim \sqrt{\frac{d\log(eD/d)}{n}}$, then with probability at least $1-\delta$, it holds that 
\[
\left\|
\begin{pmatrix}
\widehat{\omega}_n - \omega^*\\
\widehat{c}_n - c^*
\end{pmatrix}
\right\|_2 = \mathcal{O}\left( 
\sqrt{\frac{d\log(eD/d)+\log(1/\delta)}{n}}
\right),
\]
where $\mathcal{O}(\cdot)$ hides the constant depending only on $b, h, r_0, \underline{f}, \kappa, K, \|\omega^*\|_2$.
\end{theorem}
The proof of Theorem~\ref{Theorem:error:bound} is provided in Appendix~\ref{Proof:theorem:error}.
Building on Theorem~\ref{Theorem:error:bound}, we further show the out-of-sample performance and support recovery guarantees of our obtained estimator.
Given regression coefficients $(\omega,c)$, define its population newsvendor risk
\begin{equation}
\mathcal{R}(\omega,c)
=
\mathbb{E}_{(X,Y)}\big[ 
b\max(Y - \omega^\top X - c,0)
+
h\max(\omega^\top X + c - Y, 0)
\big].\label{Eq:population:risk}
\end{equation}
Our out-of-sample performance guarantee in Corollary~\ref{Theorem:out:sample:bound} quantifies the gap between the population newsvendor risk of $(\hat{\omega}_n, \hat{c}_n)$ and the optimal risk, which vanishes in the order of $d\log(eD/d)/n$.
\begin{corollary}[Out-of-Sample Performance Guarantee]\label{Theorem:out:sample:bound}
Under the same setup as in Theorem~\ref{Theorem:error:bound}, with probability at least $1-\delta$, it holds that
\[
0\le \mathcal{R}(\hat{\omega}_n,\hat{c}_n) - \min_{\|\omega\|_0\le d, c}~\mathcal{R}(\omega,c)
\le \mathcal{O}\left( 
\frac{d\log(eD/d) + \log(1/\delta)}{n}
\right),
\]
where $\mathcal{O}(\cdot)$ hides the constant depending only on $b, h, r_0, \underline{f}, \kappa, K, \|\omega^*\|_2$.
\end{corollary}
The support recovery guarantee is provided in Corollary~\ref{Theorem:support}, whose proof is provided in Appendix~\ref{Sec:corollary:support} and uses the fact that $\widehat{\omega}_n$ converges to the true sparse vector $\omega^*$, and so is its support.
\begin{corollary}[Support Recovery Guarantee]\label{Theorem:support}
Under the same setup as in Theorem~\ref{Theorem:error:bound}, and assume the sample size $n$ is sufficiently large so that with probability at least $1-\delta$, it holds that
\[
\left\|
\begin{pmatrix}
\widehat{\omega}_n - \omega^*\\
\widehat{c}_n - c^*
\end{pmatrix}
\right\|_2<\min_{j\in S^*}|\omega^*[j]|.
\]
Then, with probability at least $1-\delta$,
\[
\widehat{\omega}_n[i]\ne0, \quad\widehat{\omega}_n[j]=0,\quad \forall i\in S^*,~~j\notin S^*.
\]
\end{corollary}

\section{Numerical Study}
In this section, we evaluate the performance of the proposed methods. 
We generate synthetic data from the linear model
$Y=(\omega^*)^\top X+\varepsilon,$
where \(\varepsilon\sim\mathcal N(5,\sigma^2)\). The feature vector \(X\) follows a multivariate Gaussian distribution with mean zero and covariance matrix \(\Sigma\), where
$\Sigma[i,j]=0.5^{|i-j|}, i,j\in[D].$
Although our statistical analysis assumed that the feature vector has bounded support, we use Gaussian-distributed feature vector with unbounded support for numerical evaluation.
We further perform scalar transformation to ensure Assumption~\ref{Assumption:structure:data}\ref{Assumption:structure:data:I} holds.
The first \(d\) entries of \(\omega^*\) are nonzero and are specified as
$\omega^*[i]=(-1)^{i+1}\left(1-\frac{i-1}{d}\right), i\in[d].$
We set the unit shortage and holding costs to \(b=2\) and \(h=1\), respectively.

For baseline comparison, our experiment includes the following methods: (I) the exact algorithms based on the Big-$M$ and perspective formulations in Section~\ref{sec:exact};  
(II) the continuous-relaxation algorithm described in Algorithm~\ref{Alg:con:algorithm};  
(III) the greedy algorithm described in Algorithm~\ref{Alg:Eq:greedy}; and  
(IV) an \(\ell_1\)-based baseline obtained by replacing the \(\ell_0\)-constraint in Problem~\eqref{eq:5} with an \(\ell_1\)-regularized formulation.  
For each algorithm, we impose a time limit of \(1800\) seconds. 
The exact algorithm baseline is built on the perspective formulation except in Section~\ref{Sec:com:exact}.
The mixed-integer optimization algorithms terminate once the relative optimality gap falls below \texttt{1e-4}. Unless otherwise specified, the ridge and Lasso tuning parameters are selected by cross-validation.
We measure the performance support recovery using false discovery proportion (FDP) and non-discovery proportion (NDP), defined in \citep{bajwa2012group}:
\begin{equation}
\mathrm{FDP}(S)=\frac{|S\setminus S^*|}{|S|},\qquad
\mathrm{NDP}(S)=\frac{|S^*\setminus S|}{|S^*|},
\label{Eq:FDP:NDP}
\end{equation}
where \(S^*\) denotes the true support and \(S\) denotes the support estimated by a given algorithm. Smaller values of FDP and NDP indicate better support recovery performance. All computations are performed in 2.60GHz Intel(R) Xeon(R) Platinum 8358 CPU and 32GB main memory with Gurobi 13.0.1 solver.

\subsection{Comparison of exact algorithms}\label{Sec:com:exact}
We first compare the two exact algorithms in Section~\ref{sec:exact} under different configurations of \((D,n,d)\). Throughout this subsection, we fix \(\gamma=50\), \(M=100\), and \(\sigma^2=1\). Performance is evaluated in terms of the objective value, computational time, and FDP.
Table~\ref{table13-exact} reports the computational results for solving \eqref{eq:bigM} and \eqref{eq:misocp:final} via branch-and-bound. 
Overall, the perspective formulation outperforms the Big-$M$ formulation on most instances, and the advantage becomes more pronounced as the ambient dimension \(D\) increases. 
This observation is consistent with Theorem~\ref{thm:perspective-vs-bigM}, which shows that \eqref{eq:misocp:final} provides a stronger continuous relaxation. 
When \(D\in\{1000,2000\}\), both approaches often reach the prescribed time limit, indicating that solving the sparse newsvendor problem to global optimality can be computationally demanding in large-scale settings. 
We also observe that the FDP generally decreases as \(D\) increases in these experiments. A possible explanation is that the regularization parameter $\gamma$ is not optimally tuned such that the FDP values for different instances are not comparable.

\begin{table}[!ht] \centering
\caption{{Computational results of exact algorithms. Here we fix cardinality $d=50$.
NA indicates the algorithm fails to find any feasible solution within time limit.
}}
\footnotesize
\setlength{\tabcolsep}{2pt}\renewcommand{\arraystretch}{1.1}
\begin{tabular}{p{0.6cm} >{\centering\arraybackslash}m{1.2cm} p{1.4cm} p{1.4cm} p{1.4cm} p{1.4cm} p{1.4cm} p{1.4cm} p{1.4cm} p{1.4cm}}
\toprule
\multirow{2}{*}{$D$} & \multirow{2}{*}{$n$} & \multicolumn{4}{c}{Big-$M$} & \multicolumn{4}{c}{Perspective} \\
\cmidrule(lr){3-6} \cmidrule(lr){7-10} %
& & Objective & Gap~(\%) & Time~(s) & FDP~(\%) & Objective & Gap~(\%) & Time~(s) & FDP~(\%) \\
\midrule
 \multirow{3}{*}{100} & 500 & 10.71& 3.77  & 1800.00 & 52.00 & 10.65 & 9.69e-3 & 11.33 & 42.00\\
                      & 1000 & 10.52 & 1.02 & 1800.00 & 48.00 & 10.45 & 9.94e-3 & 50.79 & 40.00\\
                      & 5000 & 11.06 & 0.38 & 1800.00 & 26.00 &
                      10.91 & 0.14 & 1800.00 & 24.00
                      \\
\midrule
  \multirow{3}{*}{500} & 500 &
  12.38 & 100 & 1800.00 & 12.00 & 
  10.01 & 0.56 & 1800.00 & 9.70
  \\
                      & 1000 & 10.48 & 36.71 & 1800.00 & 11.00 & 
                      10.11 & 1.37 & 1800.00 & 9.70
                      \\
                      & 5000 & 10.77 & 6.85 & 1800.00 & 7.50 &
                      10.71 & 1.39 & 1800.00 & 5.20
                      \\
\midrule
  \multirow{3}{*}{1000} & 500 & 9.55 & 100.00 & 1800.00 & 5.15 & 9.17 & 1.44 & 1800.00 & 4.85\\
                      & 1000 & 
                      11.54 & 100.00 & 1800.00 & 7.63 & 
                      10.34 & 1.44 & 1800.00 & 4.15\\
                      & 5000 & 
                      10.81 & 15.52 & 1800.00 & 4.26 &
                      10.74 & 3.17 & 1800.00 & 3.68\\
\midrule
  \multirow{3}{*}{2000} & 500 &
  12.01 & 100 & 1800.00 & 6.41 &
  9.90 & 1.99 & 1800.00 & 2.46\\
                      & 1000 & 
                      12.98 & 100 & 1800.00 & 5.51 &
                      9.56 & 4.34 & 1800.00 & 2.41\\
                      & 5000 & 
                      NA & NA & 1800.00 & NA & 11.37 & 41.04 & 1800.00 & 2.46\\
\bottomrule
\end{tabular}
\label{table13-exact}
\end{table}

\subsection{Results on Support Recovery}\label{SR}
Next, we compare the variable-selection performance of all methods. 
The noise variance \(\sigma^2\) is chosen so that
$\frac{\mathrm{Var}((\omega^*)^\top X)}{\mathrm{Var}(\varepsilon)}=10,$
that is, the signal-to-noise ratio (SNR) is \(10\).

\begin{figure}[!ht]
    \centering
    \includegraphics[width=\linewidth]{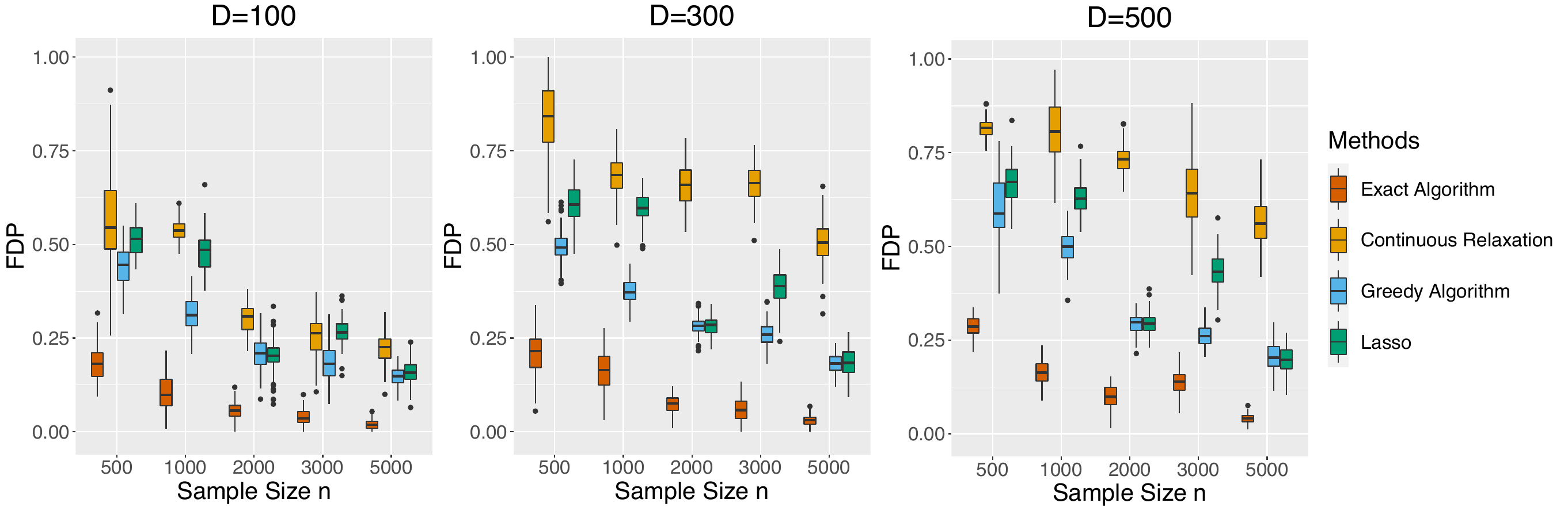}
    \caption{FDP with different choices of sample size $n$ and each subplot represents choices of dimension $D$.}
    \label{fig:FDP:D}
\end{figure}

\begin{figure}[!ht]
    \centering
    \includegraphics[width=\linewidth]{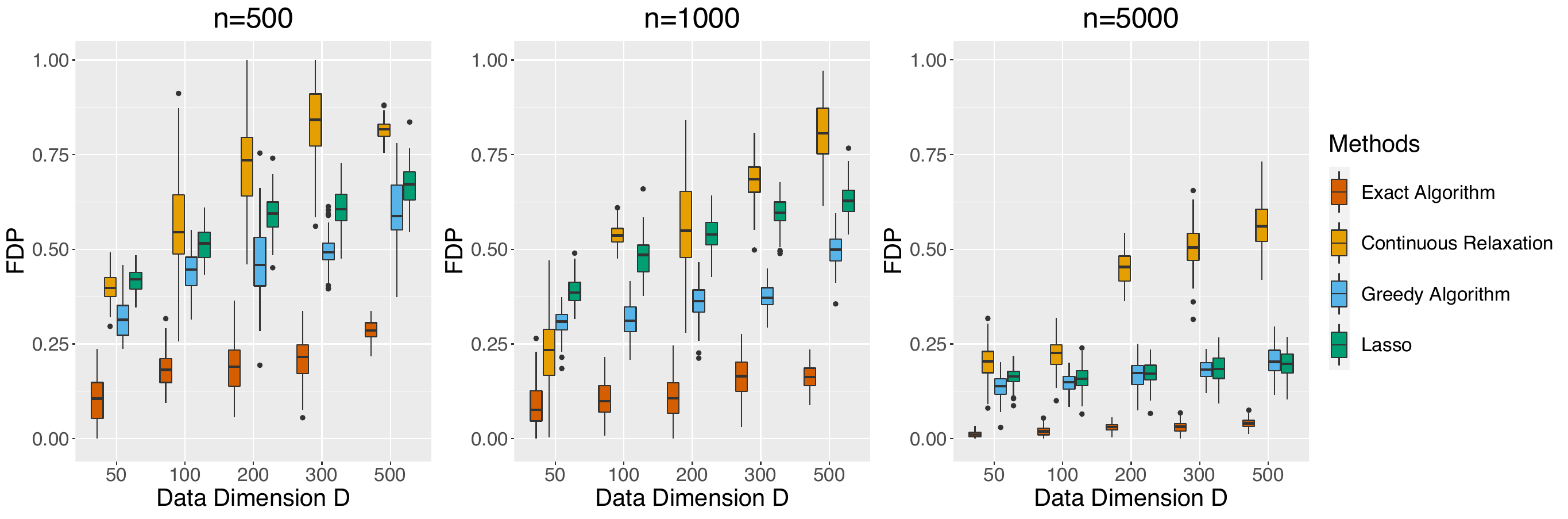}
    \caption{FDP with different choices of dimension $D$ and each subplot represents choices of sample size $n$.}
    \label{fig:FDP:n}
\end{figure}

We first consider the case in which the sparsity budget matches the true support size and is fixed at \(10\). 
In this setting, FDP and NDP coincide, so we report only FDP. 
Figures~\ref{fig:FDP:D} and~\ref{fig:FDP:n} show the FDP values under different combinations of \((n,D)\). 
The exact algorithms achieve the best support recovery performance, especially in the high-dimensional and small-sample regime. 
Among the scalable baselines, the greedy algorithm generally outperforms both the continuous-relaxation and Lasso baselines.

\begin{figure}[!ht]%
    \centering
    \includegraphics[width=1\linewidth]{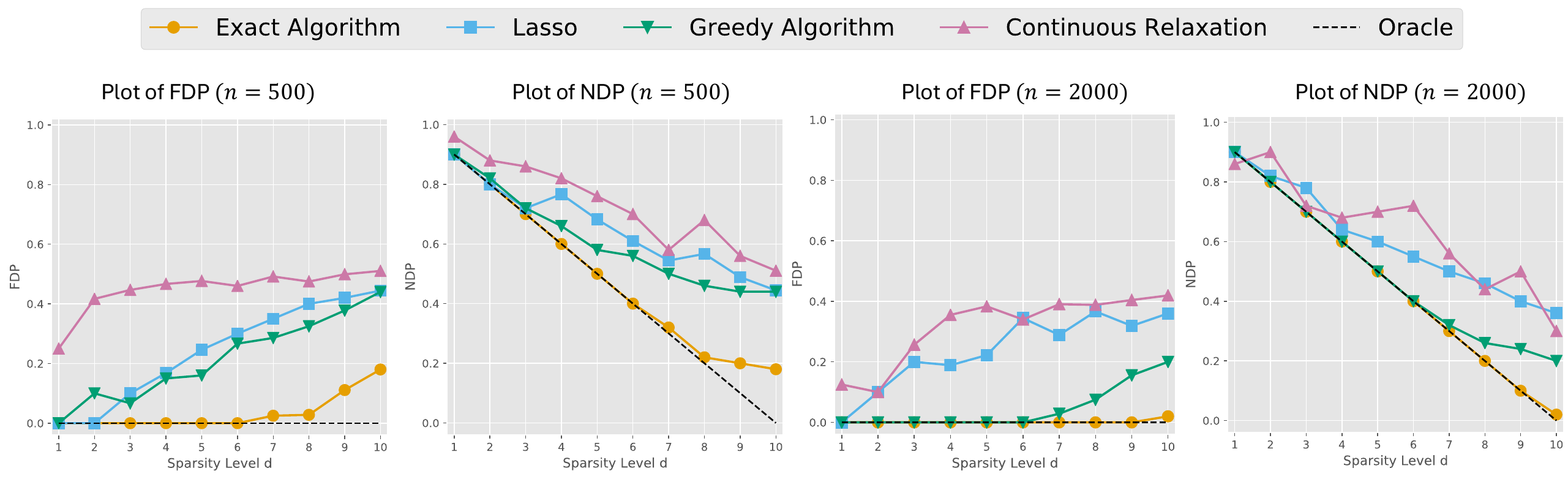}
    \caption{
    Comparison in terms of FDP and NDP for $n\in\{500,2000\}$ and $D=100$.}
    \label{fig:FDP:support}
\end{figure}

We next fix the true support size at \(10\) and vary the sparsity budget to compare the methods in terms of both FDP and NDP. 
Figure~\ref{fig:FDP:support} shows the results for \((D,n)\in\{(100,500),(100,2000)\}\). 
The black dashed line represents the ground-truth optimum of these metrics. 
As shown in the figure, exact algorithm remains close to the oracle benchmark across a wide range of sparsity budgets.

\subsection{Results on Parameter Estimation}
\label{sec:training quality}
We examine the performance of parameter estimation for various methods on problem instances where $D=100, n\in\{k\cdot100, k\in[10]\}$, and $\mathrm{SNR}\in\{0.1,1,10\}$. 
The $\ell_1$ or $\ell_2$ regularization parameters are tuned using cross-validation. 
Figure~\ref{fig:training quality} summarizes the numerical results, with the sample size $n$ on the $x$-axis and the logarithm of relative error on the $y$-axis. The relative error is defined as $|\hat{c} - c^*| / |c^*|$ or $\|\hat{\omega} - \omega^*\|_2 / \|\omega^*\|_2$, where $(\hat{\omega}, \hat{c})$ denotes the estimated parameters and $(\omega^*, c^*)$ the ground truth. All box plots are generated using $5$ independent trials.
The results indicate that, for all methods, estimation accuracy improves as the sample size $n$ increases. In particular, the exact and greedy algorithms consistently outperform the continuous relaxation and Lasso-based approaches in estimating the weight vector $\omega^*$, especially in the small-sample and low-SNR regimes.

\begin{figure}[!htbp]%
    \centering
    \includegraphics[width=1\linewidth]{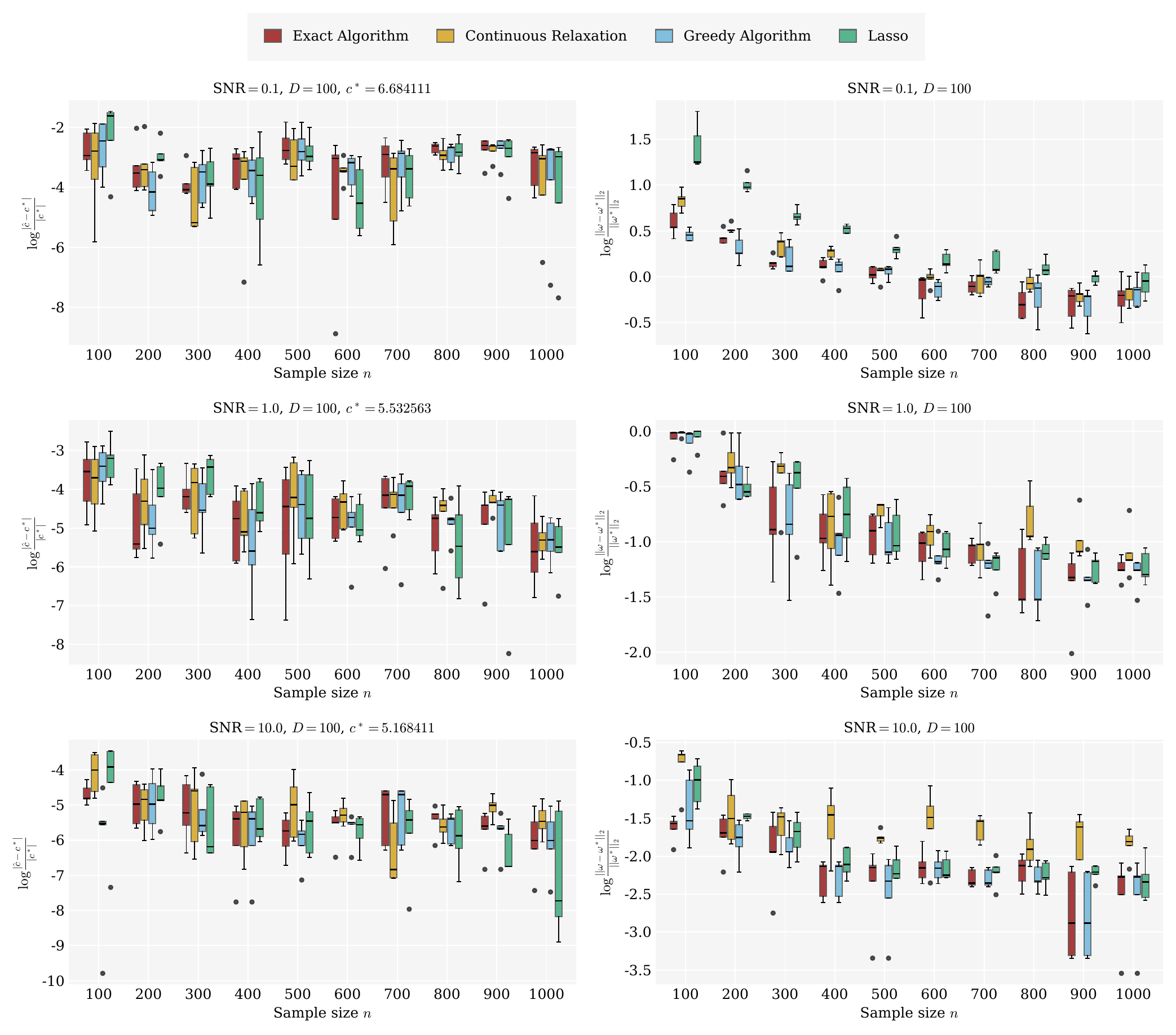}
    \caption{Comparison in terms of logarithm of relative errors for different $(n, D, \mathrm{SNR})$.}
    \label{fig:training quality}
\end{figure}

\subsection{Results on Out-of-Sample Performance}\label{Sec:out}

We examine the out-of-sample performance of various methods using the population newsvendor risk defined in \eqref{Eq:population:risk}, which is accurately estimated via sample average approximation with $10^5$ testing samples. The experimental setup follows that of the previous subsection, except that we also examine different choices of the sparsity level $d$. Recall the ground truth weight parameter $\omega^*$ is fixed with support size $|S^*|=10$. For the Lasso method, the $\ell_1$ regularization parameter is fine-tuned to achieve the desired sparsity level in the estimated parameter.
Figure~\ref{fig:generalization} summarizes the numerical results, with all error bars generated from $5$ independent trials. The left-hand-side subplots display the sample size $n$ on the $x$-axis. These results show that all methods improve with increasing $n$, while the exact and greedy algorithms consistently yield lower out-of-sample operational costs compared to the continuous relaxation and Lasso-based approaches. 
The right-hand-side subplots present the sparsity level $d$ on the $x$-axis. When the true sparsity level $|S^*|$ is underestimated ($d < |S^*|$), the exact and greedy algorithms dominate the others in terms of performance. When it is overestimated, the out-of-sample performance tends to stabilize. Together with the results in Sections~\ref{SR} and~\ref{sec:training quality}, our numerical findings demonstrate that the variable selection framework delivers satisfactory parameter estimation and out-of-sample performance while using substantially fewer covariates and yielding interpretable decisions.

\begin{figure}[!ht]%
    \centering
    \includegraphics[width=1\linewidth]{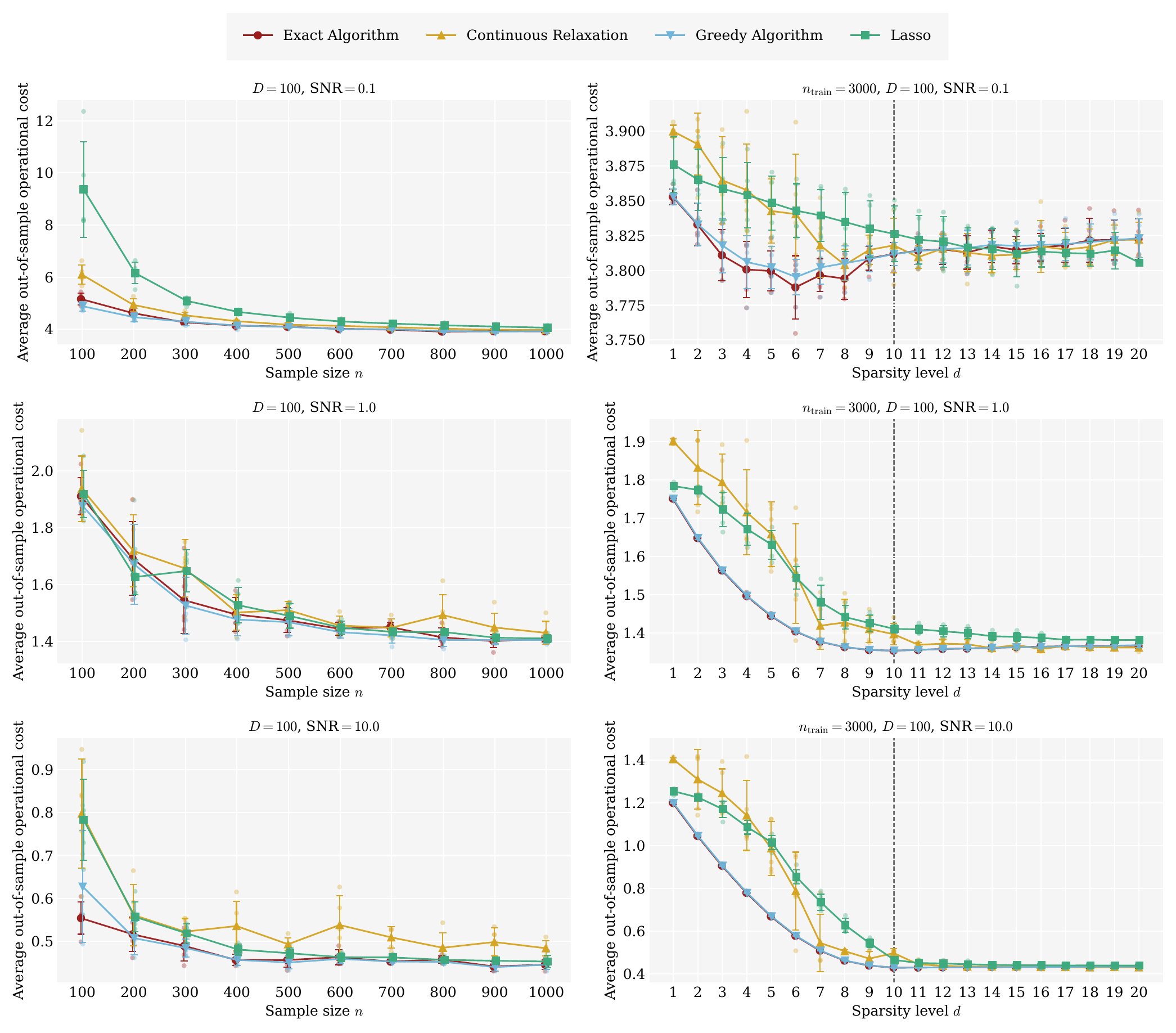}
    \caption{Comparison in terms of out-of-sample operational cost for different $(n, d, \mathrm{SNR})$.}
    \label{fig:generalization}
\end{figure}

\section{Real-world case study}
In this section, we evaluate the performance of various approaches using retail sales data from \texttt{Kaggle}\citep{Kaggle}. The real-world dataset is provided by Corporación Favorita, an Ecuadorian-based grocery retailer that sells over 200,000 products. During preprocessing, we focus on hardware products and define the response variable $Y$ as the sum of weekly unit sales across all four hardware product types. The covariate vector $X$ is constructed using standard feature transformations, including one-hot encoding and normalization. The data are restricted to the time period from \texttt{2012-12-31} to \texttt{2013-12-30}. The raw features in the dataset include: store number, city index, week index, month index, weekly averaged oil price, weekly promotion intensity, weekly transaction volume, holiday indicator, weekly number of hardware SKUs sold, and weekly number of active sales days. In addition, we incorporate the following two-way interactions as suggested in \citet{feature_rich}: (i) holiday indicator with weekly promotion intensity; (ii) store number with weekly averaged oil price; (iii) month index with city index; (iv) month index with store number; and (v) holiday indicator with store number. After preprocessing, the final dataset consists of $n = 2239$ samples and $D = 1157$ dimensions.
We split the data into 70\% training and 30\% testing sets across $20$ independent trials. The training data are used to estimate the weight vector and intercept, while the testing data are used for evaluating the (unit) out-of-sample performance.

\begin{table}[!ht] \centering
\caption{Out-of-sample operational cost and its standard deviation with different choices of penalty coefficient $\gamma$ and sparsity level $d$ (bold-faced values represent the lowest cost under the same $\gamma$)%
}
\footnotesize
\setlength{\tabcolsep}{2pt}\renewcommand{\arraystretch}{1.1}
\begin{tabular}{p{4cm} >{\centering\arraybackslash}m{2cm} p{1.4cm} p{1.4cm} p{1.4cm} p{1.4cm} p{1.4cm} p{1.4cm} p{1.4cm} p{1.4cm}}
\toprule
\multirow{2}{*}{method} & \multicolumn{2}{c}{$\gamma=0.5$} & \multicolumn{2}{c}{$\gamma=1$} & \multicolumn{2}{c}{$\gamma=5$} & \multicolumn{2}{c}{$\gamma=10$} \\

\cmidrule(lr){2-3} \cmidrule(lr){4-5} \cmidrule(lr){6-7} \cmidrule(lr){8-9}%
& mean & std  & mean & std & mean & std & mean & std \\
\midrule
Full model & \textbf{4.5359} & \textbf{0.1709} & {4.0300} & \textbf{0.1468} & 3.7305 & 0.1436 & 3.7678 & 0.1769 \\

Exact method ($d=50$) & 4.9280& 0.1842& 4.2086 &0.1672 &\textbf{3.2440} & \textbf{0.1154} & \textbf{3.2073} & \textbf{0.1246} \\

Exact method ($d=100$) &4.7706 &0.1840 &4.0875 &0.1639 &3.3469 &0.1289 & 3.3499 & 0.1540 \\

Exact method ($d=200$) & 4.6522&0.1748 &4.0460 &0.1507 &3.4718 &0.1363 & 3.5306 & 0.1586 \\

Exact method ($d=500$) & 4.5707& 0.1723& 4.0348& 0.1479&3.6941 &0.1462 & 3.7306 & 0.1793\\

Exact method ($d=800$) & 4.5374& 0.1696& \textbf{4.0281}&0.1475 &3.7296 &0.1435 & 3.7670 & 0.1757\\

\bottomrule
\end{tabular}
\label{table:real1}
\end{table}
We report the out-of-sample performance of the different methods in Table~\ref{table:real1}. Baseline approaches include the full model without variable selection, as well as our hard-cardinality-constrained model~\eqref{eq:5} under various sparsity levels $d$. An $\ell_2$ regularization term with different coefficients $\gamma$ is added into all models to mitigate overfitting. The numerical results indicate that as the sparsity budget $d$ increases, the out-of-sample performance of our variable selection framework approaches that of the full model. Moreover, with an appropriate choice of the $\ell_2$ regularization parameter and sparsity budget, our framework not only yields interpretable decisions but also attains the lowest out-of-sample operational costs among all methods considered.

\begin{table}[htbp]
\centering
\caption{Top 20 features selected using the exact algorithm with $d=100$ on the real-world data, where the entries with $x$ represent the interaction between two features}
\label{tab:realworld-top20-meanabs-d100}
\small
\setlength{\tabcolsep}{4pt}
\renewcommand{\arraystretch}{1.05}

\begin{tabular}{@{}c *{4}{P{2.8cm}}@{}}
\toprule
Rank & $\gamma=0.5$ & $\gamma=1$ & $\gamma=5$ & $\gamma=10$ \\
\midrule
1 & n\_days & n\_days & n\_days & n\_days \\
2 & n\_items & n\_items & n\_items & n\_items \\
3 & transactions & cluster\_5\_x\_oil & cluster\_5\_x\_oil & cluster\_5 \\
4 & store\_44 & store\_44\_x\_oil & store\_44\_x\_oil & store\_44 \\
5 & cluster\_5 & cluster\_5 & cluster\_5 & store\_44\_x\_oil \\
6 & store\_44\_x\_oil & store\_44 & store\_44 & cluster\_5\_x\_oil \\
7 & cluster\_5\_x\_oil & transactions & transactions & transactions \\
8 & cluster\_8 & cluster\_8 & month\_9\_x\_store\_9 & city\_Cuenca \\
9 & cluster\_8\_x\_oil & cluster\_8\_x\_oil & store\_8\_x\_oil & store\_3 \\
10 & type\_C\_x\_oil & store\_3 & store\_8 & store\_8 \\
11 & type\_C & store\_8 & store\_3 & store\_9\_x\_oil \\
12 & store\_3 & store\_8\_x\_oil & city\_Cuenca & month\_8\_x\_store\_9 \\
13 & store\_3\_x\_oil & store\_3\_x\_oil & store\_9 & month\_9\_x\_store\_9 \\
14 & store\_9 & store\_9 & type\_C & month\_2\_x\_store\_44 \\
15 & store\_8 & type\_C\_x\_oil & store\_9\_x\_oil & type\_C \\
16 & store\_9\_x\_oil & store\_9\_x\_oil & type\_C\_x\_oil & store\_8\_x\_oil \\
17 & store\_8\_x\_oil & type\_C & store\_3\_x\_oil & city\_Cuenca\_x\_oil \\
18 & city\_Cuenca & holiday\_x\_store\_44 & month\_8\_x\_store\_9 & month\_8\_x\_store\_44 \\
19 & holiday\_x\_store\_44 & store\_41\_x\_oil & city\_Cuenca\_x\_oil & store\_48\_x\_oil \\
20 & city\_Quito & store\_41 & cluster\_10\_x\_oil & store\_9 \\
\bottomrule
\end{tabular}
\end{table}

\begin{table}[ht!]
\centering
\caption{Top 20 features selected using the exact algorithm with $d=50$ on the real-world data}
\label{tab:realworld-top20-meanabs-d50}
\small
\setlength{\tabcolsep}{4pt}
\renewcommand{\arraystretch}{1.05}

\begin{tabular}{@{}c *{4}{P{2.8cm}}@{}}
\toprule
Rank & $\gamma=0.5$ & $\gamma=1$ & $\gamma=5$ & $\gamma=10$ \\
\midrule
1 & n\_days & n\_days & n\_days & n\_days \\
2 & n\_items & n\_items & n\_items & n\_items \\
3 & transactions & cluster\_5\_x\_oil & cluster\_5\_x\_oil & store\_44\_x\_oil \\
4 & store\_44 & store\_44\_x\_oil & store\_44\_x\_oil & cluster\_5\_x\_oil \\
5 & cluster\_5 & store\_44 & cluster\_5 & cluster\_5 \\
6 & store\_44\_x\_oil & transactions & store\_44 & store\_44 \\
7 & cluster\_5\_x\_oil & cluster\_5 & transactions & transactions \\
8 & type\_C\_x\_oil & type\_C\_x\_oil & store\_8 & city\_Cuenca \\
9 & type\_C & type\_C & store\_8\_x\_oil & store\_3 \\
10 & cluster\_8 & cluster\_8 & city\_Cuenca & store\_8 \\
11 & cluster\_8\_x\_oil & cluster\_8\_x\_oil & store\_3 & store\_8\_x\_oil \\
12 & store\_3 & store\_9 & store\_9\_x\_oil & month\_9\_x\_store\_9 \\
13 & store\_3\_x\_oil & store\_9\_x\_oil & city\_Cuenca\_x\_oil & month\_8\_x\_store\_9 \\
14 & store\_9 & store\_8 & store\_9 & store\_9\_x\_oil \\
15 & store\_9\_x\_oil & store\_8\_x\_oil & type\_C & month\_2\_x\_store\_44 \\
16 & store\_8 & store\_3 & store\_3\_x\_oil & store\_3\_x\_oil \\
17 & store\_8\_x\_oil & store\_3\_x\_oil & month\_9\_x\_store\_9 & store\_9 \\
18 & city\_Cuenca & holiday\_x\_store\_44 & type\_C\_x\_oil & month\_8\_x\_store\_44 \\
19 & city\_Cuenca\_x\_oil & city\_Cuenca & cluster\_8 & type\_C\_x\_oil \\
20 & city\_Quito & city\_Cuenca\_x\_oil & store\_48\_x\_oil & store\_48\_x\_oil \\
\bottomrule
\end{tabular}
\end{table}

We report the top $20$ selected features under different configurations of $(\gamma, d)$ in Tables~\ref{tab:realworld-top20-meanabs-d100} and \ref{tab:realworld-top20-meanabs-d50}. The features are ranked according to their averaged absolute weight values in $20$ independent trials. 
From these selected features, we draw the following operational insights: (I) The number of days per week on which customers purchase hardware products ({n\_days}) and the weekly number of active SKUs ({n\_items}) emerge as the most persistent features. A possible explanation is that customer demand is largely driven by weekend shopping habits and the breadth of product assortment available in stores. (II) The historical weekly transaction volume ({transactions}) consistently appears as a strong predictor, as it is highly correlated with future customer demand. (III) The location and type of the store~(such as those represented by the store\_3, 8, 9, 44, C, city\_Cuenca) play a critical role in predicting hardware sales. (IV) The interaction terms between store 9 and the months of August and September, as well as between store 44 and February and August, suggest that the effects of store location and time on operational outcomes are heterogeneous, with certain stores exhibiting pronounced seasonal demand spikes.

\section{Conclusion}
In this paper, we studied variable selection for the feature-based newsvendor problem under a hard cardinality constraint on the number of selected covariates. We establish both computational algorithms and statistical performance guarantees of this framework. The numerical results illustrate the trade-off among computational effort, support-recovery accuracy, and out-of-sample operating cost.

Several directions remain open for future research. First, it is promising to extend this variable selection framework for nonlinear policy classes, such as the kernel-based decision rules~\citep{wang2023variable}. Second, it is an interesting direction to design more efficient algorithms and hyperparameter selection procedures~\citep{cory2025stability, cory2023optimal}. Finally, extending the statistical analysis to misspecified, heavy-tailed, dependent, or distributionally robust settings would broaden the applicability of our framework.

\bibliographystyle{informs2014} 
\bibliography{shortbib}

\ECSwitch

\ECHead{Appendix for \emph{``Variable Selection for Feature-Based Newsvendor''}}

\section{Proof of Proposition~\ref{prop:hardness}}\label{Proof:Proposition:hardness}

We prove the NP-hardness result by making the reduction from \texttt{Exact Cover by 3-Sets}~(X3C) to Problem~\eqref{eq:5}.
The formal description of (X3C) is as follows. 
Let $Q$ be a ground set consisting of $m:=3d$ elements, and $\mathcal{C}=\{C_1,\ldots,C_D\}$ be the collection of subsets of $Q$ such that $|C_i|=3$ for $i\in[D]$. Problem~(X3C) aims to determine whether there exists a subset of $\mathcal{C}$, denoted as $\mathcal{C}'$ such that every element in $Q$ occurs exactly once in $\mathcal{C}'$.

Denote by $A\in\mathbb{R}^{m\times D}$ the incidence matrix of Problem~(X3C) such that $A[i,j]=1$ if the $i$-th element of $Q$ belongs to the subset $C_j$ and otherwise $A[i,j]=0, \forall i,j.$
We define a scalar $M_0$ such that $M_0>\frac{m}{3\gamma(b\land h)}$. We construct a special instance of Problem~\eqref{eq:5} by taking $n=2m$ and the $i$-th observation as
\[
(x_i, y_i) = \left\{
\begin{aligned}
(M_0 A[i,:]^\top,M_0)&\quad \text{if }1\le i\le m,\\
(-M_0 A[i-m,:]^\top, -M_0)&\quad \text{if }m+1\le i\le n.
\end{aligned}
\right.
\]
For weight vector $\omega\in \mathbb{R}^{D}$, define the error vector $r=\mathbf{1}-A\omega\in\mathbb{R}^m$.
For any scalar $x$, define $x_+=\max\{x,0\}$ and $x_-=\max\{-x,0\}$.
Define $b\land h=\min(b,h)$.
Thus, the fitting error part can be bounded as 
\begin{align*}
&\frac{1}{n}\sum_{i=1}^n\Big[ 
b \cdot \max\{0, y_i - \omega^\top x_i-c\} + h \cdot \max\{0, \omega^\top x_i+c - y_i\}
\Big]\\
=&\frac{1}{2m}\sum_{i=1}^m
\Big[ 
b \cdot (M_0 r[i]-c)_+ + h \cdot (M_0 r[i]-c)_-
+
b \cdot (-M_0 r[i]-c)_+ + h \cdot (-M_0 r[i]-c)_-
\Big]\\
\ge&\frac{b\land h}{2m}\sum_{i=1}^m\Big[ 
|M_0 r[i]-c| + |M_0 r[i]+c|
\Big]\\
\ge&\frac{(b\land h)M_0}{m}\|r\|_1,
\end{align*}
where the last inequality is based on the relation $|u+v| + |u-v|\ge 2|u|$. 

Let $(\omega,c)$ be a feasible solution to Problem~\eqref{eq:5} with $\|\omega\|_0\leq d$.
We first show that $F(\omega,c)\ge d/(2\gamma)$.
\begin{itemize}
    \item 
Assume that $\|\omega\|_2^2\leq d$, then %
\begin{equation}
    |\mathbf{1}^\top\omega|=\left|
\sum_{i:~\omega[i]\ne0}\omega[i]
\right|\le 
\sqrt{
\sum_{i:~\omega[i]\ne0}1
}\sqrt{
\sum_{i:~\omega[i]\ne0}\omega^2[i]
}\le \sqrt{d}\|\omega\|_2,\label{Eq:1:omega:d}
\end{equation}
where the first inequality uses Cauchy-Schwarz inequality.
It follows that
\begin{align*}
d - \|\omega\|_2^2&\le d - \frac{|\mathbf{1}^\top\omega|^2}{d}=
(d-\mathbf{1}^\top\omega)\cdot\left(
1+\mathbf{1}^\top\omega/d
\right)\\
&\le 2(d-\mathbf{1}^\top\omega)\\
&=\frac{2}{3}\textbf{1}^\top r \le \frac{2}{3}\|r\|_1,
\end{align*}
where the first inequality is by \eqref{Eq:1:omega:d}, the second inequality is because $|\mathbf{1}^\top\omega|\le \sqrt{d}\|\omega\|_2\le d$, and the last equality is because each column of $A$ has exactly three ones and 
\begin{equation}
\textbf{1}^\top r = m - \textbf{1}^\top A\omega = m - (A^\top \textbf{1})^\top \omega 
=
3d - 3(\mathbf{1}^\top\omega).\label{Eq:1:r}
\end{equation}
then, the objective
\begin{equation}
F(\omega,c)\ge \frac{(b\land h)M_0}{m}\|r\|_1 + \frac{1}{2\gamma}\left( 
d - \frac{2}{3}\|r\|_1
\right)\ge \frac{d}{2\gamma},\label{Eq:F:omega:c:bound}
\end{equation}
where the last inequality is by the choice of $M_0$.
\item
Otherwise, for $\|\omega\|_2^2>d$, the objective value 
\[
F(\omega,c)\ge  \frac{1}{2\gamma}\|\omega\|_2^2 > \frac{d}{2\gamma}.
\]
\end{itemize}

We now show that Problem (X3C) can be reduced to Problem~\eqref{eq:5}. Assume we solve Problem~\eqref{eq:5} with optimal value upper bounded by $\frac{d}{2\gamma}$ and denote its optimal solution as $(\omega, c)$. This happens only if $\|\omega\|_2^2\le d$, and the inequalities in \eqref{Eq:F:omega:c:bound} all hold with equalities. It implies that $r=0=\textbf{1}-A\omega$. By \eqref{Eq:1:r}, it holds that $\textbf{1}^\top\omega = d$.
By \eqref{Eq:1:omega:d}, it holds that $\|\omega\|_2^2=d$. 
Then, all inequalities in \eqref{Eq:1:omega:d} always hold with equalities, and therefore $\omega$ has exactly $d$ nonzero entries.
The exact cover of (X3C) is given by $\{C_j:~\omega[j]\ne0\}$.

Conversely, suppose Problem~(X3C) has an exact cover.
Let $\omega\in\mathbb{R}^D$ be the associated weight vector such that $\omega[i]=1$ if the $i$-th element of $\mathcal{C}$ is selected to form the cover and otherwise $\omega[i]=0$. Then the incidence matrix $A$ satisfies that $r:=\textbf{1}-A\omega=0$ and $\omega$ satisfies $\|\omega\|_0=d$ and $\|\omega\|_2^2=d$. The objective value of the feasible solution $(\omega,0)$ to Problem~\eqref{eq:5} equals
\begin{align*}
F(\omega,0)=&\frac{1}{2m}\sum_{i=1}^m
\Big[ 
b \cdot (M_0 r[i])_+ + h \cdot (M_0 r[i])_-
+
b \cdot (-M_0 r[i])_+ + h \cdot (-M_0 r[i])_-
\Big]+\frac{d}{2\gamma}\\
=&\frac{d}{2\gamma}.
\end{align*}
It justifies that the optimal value of Problem~\eqref{eq:5} is bounded by $d/(2\gamma)$. The proof is completed.

\section{Proof of Theorem \ref{thm:perspective-vs-bigM}}\label{Appendix:thm:perspective-vs-bigM}
Both Big-$M$ and perspective reformulation can be written as the following form of optimization:
\begin{equation}
\begin{aligned}
\min_{\omega,c,s,\mu}\quad
& \frac{1}{n}\sum_{i=1}^n\Big[
b\max\{0,\, y_i-\omega^\top x_i-c\}
+
h\max\{0,\, \omega^\top x_i+c-y_i\}
\Big]
+\frac{1}{2\gamma}\|\mu\|_1\\
\mbox{s.t.}&\quad (\omega,c,s,\mu)\in \mathcal{V}^{\text{Big-M}}\text{ or }\mathcal{V}^{\text{Pers}},\\
&\quad \sum_{j\in[D]}s[j]\le d,
\end{aligned}
\end{equation}
where
\[
\begin{aligned}
\mathcal{V}^{\text{Big-M}}& =\left\{
(\omega, s, \mu,c)\in\mathbb{R}^D\times\{0,1\}^D\times\mathbb{R}_+^D\times\mathbb{R}:
\begin{array}{l}
\omega[i]^2\le \mu[i], i\in[D]\\
|\omega[i]|\le M[i]\,s[i], i\in[D]
\end{array}
\right\},\\
\mathcal{V}^{\text{Pers}} & =\left\{
(\omega, s, \mu,c)\in\mathbb{R}^D\times\{0,1\}^D\times\mathbb{R}_+^D\times\mathbb{R}:
\begin{array}{l}
\omega[i]^2\le \mu[i] s[i], i\in[D]\\
|\omega[i]|\le M[i]\,s[i], i\in[D]
\end{array}
\right\}.
\end{aligned}
\]
These two problems are equivalent because $\mathcal{V}^{\text{Big-M}}=\mathcal{V}^{\text{Pers}}$.
Indeed, when the binary variable $s[i]=0$, it holds that $\omega[i]=0$ and the first constraints appeared in $\mathcal{V}^{\text{Big-M}}$ and $\mathcal{V}^{\text{Pers}}$ are redundant. Otherwise, the first constraints in $\mathcal{V}^{\text{Big-M}}$ and $\mathcal{V}^{\text{Pers}}$ are equivalent.

For Part~(II), if considering the continuous relaxation of these two problems, the binary constraints $s[i]\in\{0,1\}$ in $\mathcal{V}^{\text{Big-M}}$ and $\mathcal{V}^{\text{Pers}}$ are replaced by $s[i]\in[0,1]$ for $i\in[D]$.
Denoted the continuous relaxation of these two sets as $\overline{\mathcal{V}}^{\text{Big-M}}$ and $\overline{\mathcal{V}}^{\text{Pers}}$.
It holds that $\overline{\mathcal{V}}^{\text{Pers}}\subseteq \overline{\mathcal{V}}^{\text{Big-M}}$ as $\omega[i]^2\leq \mu[i]s[i]$ implies $\omega[i]^2\leq \mu[i]$.
Therefore, the optimal value of continuous relaxation of perspective reformulation is larger than that of Big-$M$ reformulation.

\section{Proof of Theorem \ref{Theorem:approximation:guarantee}}\label{proof:thm:app}

For Part (I), we bound the support size of solution in Algorithm~\ref{Alg:con:algorithm}. Let $(\widehat{\omega}, \widehat{c}, \widehat{s}, \widehat{\mu})$ be the optimal solution to the continuous relaxation of~\eqref{eq:misocp:final}.
Let $(\tilde{\omega}, \tilde{c},\tilde{s})$ be the output of Algorithm~\ref{Alg:con:algorithm}.
The support size of $\tilde{\omega}$ is bounded by the sum of random variables $N:=\sum_{j\in[D]}\tilde{s}[j]$, where $\tilde{s}[j]\sim\mathrm{Bernoulli}(\widehat{s}[j])$. Its expected value 
\[
\mathbb{E}\left[ 
N
\right]=\sum_{j\in[D]}\widehat{s}[j]\le d.
\]
For $t>0$, we compute the moment-generating function of $N$:
\begin{align*}
M_{N}(t)&=\mathbb{E}[e^{tN}] = \prod_{j\in[D]}~\mathbb{E}[e^{t\tilde{s}[j]}]\\
&=\prod_{j\in[D]}~\left( 
1+(\widehat{s}[j]e^t - \widehat{s}[j])
\right)\le 
\prod_{j\in[D]}~\exp\left( 
\widehat{s}[j]e^t - \widehat{s}[j]
\right)\\
&=\exp\left( 
(e^t-1)\sum_{j\in[D]}\widehat{s}[j]
\right)\le \exp(d(e^{t}-1)),
\end{align*}
where the first inequality is based on the relation that $1+x\le e^x, \forall x$, and the last inequality is because $\sum_{j\in[D]}\widehat{s}[j]\le d$.
For $\Delta>0$, applying the Markov inequality gives
\begin{align*}
\mathbb{P}\left( 
N\ge (1+\Delta)d
\right)&=\mathbb{P}\left( 
e^{tN}\ge e^{t(1+\Delta)d}
\right)\le \frac{M_N(t)}{e^{t(1+\Delta)d}}\\
&\le  \exp(d(e^{t}-1) - t(1+\Delta)d).
\end{align*}
We choose $t$ such that the upper bound above is as sharp as possible, yielding the optimal choice $t=\log(1+\Delta)$.
Then, it holds that
\[
\mathbb{P}\left( 
N\ge (1+\Delta)d
\right)
\le \exp(
-d((1+\Delta)\log(1+\Delta)-\Delta)
).
\]
For $\Delta\ge0$, it is easy to verify that $\log(1+\Delta)\ge \frac{2\Delta}{1+\Delta}$, and then
\[
\mathbb{P}\left( 
N\ge (1+\Delta)d
\right)
\le \exp\left(
-\frac{d\Delta^2}{2+\Delta}
\right).
\]
Taking $\Delta=\frac{\xi}{2}+\sqrt{2\xi+\frac{\xi^2}{4}}$ with $\xi:=\log(2/\delta)/d$ yields
\begin{equation}
\mathbb{P}\left( 
\|\tilde{\omega}\|_0\ge (1+\Delta)d
\right)
\le 
\mathbb{P}\left( 
N\ge (1+\Delta)d
\right)
\le \frac{\delta}{2}.\label{Eq:concentration:result}
\end{equation}
We replace this bound with a conservative one~($\Delta=\log(2/\delta)/d+\sqrt{2\log(2/\delta)/d}$) to simplify the notation in the theorem statement.

For Part~(II), we first reformulate the continuous relaxation of Problem~\eqref{eq:misocp:final}:
\begin{equation}
\begin{aligned}
\min_{\omega,c,s,\mu}\quad
& \frac{1}{n}\sum_{i=1}^n\Big[
b\max\{0,\, y_i-\omega^\top x_i-c\}
+
h\max\{0,\, \omega^\top x_i+c-y_i\}
\Big]
+\frac{1}{2\gamma}\|\mu\|_1\\
\mbox{s.t.}&\quad (\omega,c,s,\mu)\in\mathbb{R}^D\times\mathbb{R}\times[0,1]^D\times\mathbb{R}_+^D,\\
&\quad \sum_{j\in[D]}s[j]\le d,\\
&\quad \omega[i]^2\le \mu[i] s[i], i\in[D].
\end{aligned}
\end{equation}
The variable $\mu$ above can be eliminated by taking $\mu[i]=\frac{\omega[i]^2}{s[i]}$~(by default, we take $\mu[i]=0$ if $\omega[i]=s[i]=0$), yielding the reformulation
\begin{equation}
\begin{aligned}
\min_{\omega,c,s}\quad
& \frac{1}{n}\sum_{i=1}^n\Big[
b\max\{0,\, y_i-\omega^\top x_i-c\}
+
h\max\{0,\, \omega^\top x_i+c-y_i\}
\Big]
+\sum_{i=1}^D\frac{\omega[i]^2}{2\gamma s[i]}\\
\mbox{s.t.}&\quad (\omega,c,s)\in\mathbb{R}^D\times\mathbb{R}\times[0,1]^D,\\
&\quad \sum_{j\in[D]}s[j]\le d.
\end{aligned}
\end{equation}
By fixing $s$ and taking the duality of the remaining optimization problem, we obtain the min-max reformulation of the continuous relaxation:
\begin{equation}
\min_{
\substack{s\in[0,1]^D,\\
\sum_{j\in[D]}s[j]\le d
}}~\left\{g(s)=\max_{z[i]\in[-b,h], i\in[n],\sum_{i=1}^n z[i]=0} -z^\top
    \left( 
\frac{\gamma}{2n^2}
X\mathrm{Diag}(s)X^\top
    \right)z-\frac{1}{n}z^\top y\right\}.
\end{equation}
It is also clear that $\text{optval}\eqref{eq:5}=\min~\{g(s):~s\in\{0,1\}^D, \sum_{j\in[D]}s[j]\le d\}$.
By definition, 
\begin{equation}
F(\tilde{\omega}, \tilde{c})-\text{optval}\eqref{eq:5}\leq F(\tilde{\omega}, \tilde{c})-g(\widehat{s})= g(\widetilde{s})-g(\widehat{s}),\label{Eq:gap:optval:tilde}
\end{equation}
where the first inequality is because $\widehat{s}$ is the optimal solution obtained from continuous relaxation, and the second inequality is by strong duality.

Next, we provide the upper bound on $g(\tilde{s})-g(\widehat{s})$.
Define the inner objective of $g(s)$ as 
\[
\mathfrak{g}(s,z):=-z^\top
    \left( 
\frac{\gamma}{2n^2}
X\mathrm{Diag}(s)X^\top
    \right)z-\frac{1}{n}z^\top y.
\]
Let $\tilde{z}$ be the maximizer of the inner problem of $g(\tilde{s})$, then 
\begin{equation}\label{Eq:g:simplify}
\begin{aligned}
g(\tilde{s})-g(\widehat{s})& \leq \mathfrak{g}(\tilde{s},\tilde{z}) - \mathfrak{g}(\widehat{s}, \tilde{z})\\
&=-\frac{\gamma}{2n^2}\tilde{z}^\top\left( 
X\mathrm{Diag}(\tilde{s} - \widehat{s})X^\top
\right)\tilde{z}\\
&\le \frac{\gamma}{2n^2}\|\tilde{z}\|_2^2
\lambda_{\max}\Big(
X\mathrm{Diag}(\widehat{s} - \tilde{s} )X^\top
\Big)
\\
&\le \frac{\gamma(b\lor h)^2}{2n}\lambda_{\max}\Big(
X\mathrm{Diag}(\widehat{s} - \tilde{s} )X^\top
\Big),
\end{aligned}
\end{equation}
where the first inequality is because $\tilde{z}$ is a feasible solution (but perhaps suboptimal) to the inner problem of $g(\widehat{s})$, the second inequality is based on the Cauchy-Schwarz inequality, and the last one is because $\|\tilde{z}\|_2^2\le n(b\lor h)^2$.
It suffices to build the concentration bound on $\lambda_{\max}\Big(
X\mathrm{Diag}(\widehat{s} - \tilde{s} )X^\top
\Big)=
\lambda_{\max}\Big( 
\sum_{i\in[D]}R_i
\Big)$ with $R_i=(\widehat{s}[i] - \tilde{s}[i])X_{:,i}X_{:,i}^\top$.
Following the similar argument as in \citep[Proof of Lemma~3]{xie2020scalable}, it can be shown that $\mathbb{E}[R_i]=0, \lambda_{\max}(R_i)\le \theta_1$ and $\|\sum_{i\in[D]}\mathbb{E}[R_i^2]\|_2\le \theta_1\theta_d$.
Applying the matrix concentration bound in \citep[Theorem~1.4]{Tropp2010UserFriendlyTB} yields
\begin{equation}
\mathbb{P}\Big\{
\lambda_{\max}\Big(
X\mathrm{Diag}(\widehat{s} - \tilde{s} )X^\top
\Big)\ge t
\Big\}\le n\exp\left(
-\frac{t^2}{2\theta_1\theta_d + 2/3{\theta_1}t}
\right).\label{Eq:random:matrix}
\end{equation}
Taking $t=\frac{2}{3}{\theta_1}\log(2n/\delta) + \sqrt{2\theta_1\theta_d\log(2n/\delta)}$ ensures the right-hand-side of \eqref{Eq:random:matrix} being bounded by $\delta/2$.
Combining the inequalities in \eqref{Eq:gap:optval:tilde}, \eqref{Eq:g:simplify}, and \eqref{Eq:random:matrix} yields
\begin{equation*}
\begin{aligned}
&\mathbb{P}\left\{
F(\tilde{\omega}, \tilde{c})-\text{optval}\eqref{eq:5}
\le 
 \frac{\gamma(b\lor h)^2}{2n}\cdot\left( 
\frac{2}{3}{\theta_1}\log(2n/\delta) + \sqrt{2\theta_1\theta_d\log(2n/\delta)}
 \right)
\right\}\\
\ge&\mathbb{P}\left\{
\lambda_{\max}\Big(
X\mathrm{Diag}(\widehat{s} - \tilde{s} )X^\top
\Big)\le\frac{2}{3}{\theta_1}\log(2n/\delta) + \sqrt{2\theta_1\theta_d\log(2n/\delta)}
\right\}\\
\ge&1-\frac{\delta}{2}.
\end{aligned}
\end{equation*}
The choice of $\gamma$ in the main statement of Theorem~\ref{Theorem:approximation:guarantee} further ensures that 
\begin{equation}
\mathbb{P}\left\{
F(\tilde{\omega}, \tilde{c})-\text{optval}\eqref{eq:5}
\le 
\epsilon
\right\}\ge 1-\delta/2.\label{Eq:diff:F:optval}
\end{equation}
Applying the union bound of the concentration results in \eqref{Eq:concentration:result} and \eqref{Eq:diff:F:optval} gives the desired result.

\section{Proof of Theorem~\ref{Theorem:error:bound}}\label{Proof:theorem:error}
Let us define the empirical cost 
\[
\cR_n(\omega, c)=\frac{1}{n}\sum_{i=1}^n\Big[b \cdot \max\{0, y_i - \omega^\top x_i-c\} + h \cdot \max\{0, \omega^\top x_i+c - y_i\}\Big].
\]
and the population cost
\[
\cR(\omega, c) = \mathbb{E}_{(X,Y)}\left[ 
b \cdot \max\{0, Y - \omega^\top X-c\} + h \cdot \max\{0, \omega^\top X+c - Y\}
\right].
\]
We rely on the following two technical lemmas that are helpful to show the proof of Theorem~\ref{Theorem:error:bound}.
\begin{lemma}[Global Risk Lower Bound]\label{Lemma:risk:lower}
Under Assumption~\ref{Assumption:structure:data}, for any $\Delta:=(u^\top,t)^\top\in\mathbb{R}^{D+1}$ with $\|u\|_0\le 2d$, it holds that
\[
\cR(\omega^* + u, c^* + t) - \cR(\omega^*, c^*)\ge \frac{(b+h)\underline{f}\kappa r_0}{2(r_0 + K\|\Delta\|_2)}\cdot \|\Delta\|_2^2.
\]
\end{lemma}
\proof{Proof.}
By definition, 
\begin{align*}
&\cR(\omega^* + u, c^* + t) - \cR(\omega^*, c^*)\\
=&(b+h)\mathbb{E}_{(X,Y)}\left[ 
\rho_{\tau}(Y - (\omega^*+u)^\top X - (c^*+t)) - \rho_{\tau}(Y - (\omega^*)^\top X - c^*)
\right]\\
=&(b+h)\mathbb{E}_{(\varepsilon,X)}\left[ 
\rho_{\tau}(\varepsilon - \Delta^\top \widetilde{X}) - \rho_{\tau}(\varepsilon)
\right],
\end{align*}
where the first equality is by the definition of quantile function $\rho_{\tau}(u)=u(\tau - \textbf{1}_{\{u<0\}})$, and the second one is based on the model assumption~\eqref{Eq:model} and the expressions that $\widetilde{X}=(X^\top, 1)^\top, \Delta=(u^\top,t)^\top$.
Based on the Knight's identity~\citep{Knight1998}, it holds that
\begin{equation}
\begin{aligned}
&\mathbb{E}_{(\varepsilon,X)}\left[ 
\rho_{\tau}(\varepsilon - \Delta^\top \widetilde{X}) - \rho_{\tau}(\varepsilon)
\right]\\
=&\mathbb{E}_{(\varepsilon,X)}\left[ 
- \Delta^\top \widetilde{X}(\tau-\textbf{1}_{\{\varepsilon<0\}}) + \int_0^{\Delta^\top \widetilde{X}}\Big( 
\textbf{1}_{\{\varepsilon\le s\}} - \textbf{1}_{\{\varepsilon<0\}}
\Big)\diff s
\right]\\
=&\mathbb{E}_{(\varepsilon,X)}\left[ \int_0^{\Delta^\top \widetilde{X}}\Big( 
\textbf{1}_{\{\varepsilon\le s\}} - \textbf{1}_{\{\varepsilon<0\}}
\Big)\diff s
\right]=
\mathbb{E}_{X}\left[
\int_0^{\Delta^\top \widetilde{X}}\Big(F_{\varepsilon}(s) - F_{\varepsilon}(0)\Big)\diff s
\right],
\end{aligned}\label{Eq:rho:diff}
\end{equation}
where the second and last equalities are due to the assumption that $\varepsilon$ and $X$ are independent together with the definition of the cumulative distribution function~(cdf) $F_{\varepsilon}(s):=\mathbb{P}(\varepsilon\le s)$.

We provide lower bound on $\int_0^{\Delta^\top \widetilde{X}}\Big(F_{\varepsilon}(s) - F_{\varepsilon}(0)\Big)\diff s$ as follows.
When $|\Delta^\top \widetilde{X}|\le r_0$, by Assumption~\ref{Assumption:structure:data}\ref{Assumption:structure:data:II}, 
\[
\int_0^{\Delta^\top \widetilde{X}}\Big(F_{\varepsilon}(s) - F_{\varepsilon}(0)\Big)\diff s
\ge \int_0^{|\Delta^\top \widetilde{X}|}\underline{f}s\diff s=\frac{\underline{f}(\Delta^\top \widetilde{X})^2}{2},
\]
where the inequality is because $F_{\varepsilon}(s) - F_{\varepsilon}(0) = \int_0^sf_{\varepsilon}(t)\diff t\ge \int_0^s\underline{f}\diff t=\underline{f}s$ for $s\ge0$.
When $\Delta^\top \widetilde{X}>r_0>0$, it holds that
\begin{align*}
&\int_0^{\Delta^\top \widetilde{X}}\Big(F_{\varepsilon}(s) - F_{\varepsilon}(0)\Big)\diff s
\ge
\int_0^{r_0}\Big(F_{\varepsilon}(s) - F_{\varepsilon}(0)\Big)\diff s
+
\int_{r_0}^{\Delta^\top \widetilde{X}}\Big(F_{\varepsilon}(r_0) - F_{\varepsilon}(0)\Big)\diff s\\
\ge&\frac{\underline{f}r_0^2}{2} + (\Delta^\top \widetilde{X} - r_0)\underline{f}r_0\ge \frac{\underline{f}r_0\Delta^\top \widetilde{X}}{2},
\end{align*}
where the first inequality is by separating the integral on $[0, \Delta^\top\widetilde{X}]$ into two parts and by the monotonicity of the cdf, the second inequality is by Assumption~\ref{Assumption:structure:data}\ref{Assumption:structure:data:II}, and the last one is by the relation $\Delta^\top \widetilde{X}>r_0$.
It can be shown using the similar argument that $\int_0^{\Delta^\top \widetilde{X}}\Big(F_{\varepsilon}(s) - F_{\varepsilon}(0)\Big)\diff s\ge -\underline{f}r_0\Delta^\top \widetilde{X}/2$ for $\Delta^\top \widetilde{X}<-r_0<0$.
Combining these cases yields the global lower bound
\[
\int_0^{\Delta^\top \widetilde{X}}\Big(F_{\varepsilon}(s) - F_{\varepsilon}(0)\Big)\diff s
\ge \underline{f}|\Delta^\top \widetilde{X}|/2\cdot\min\{
r_0, |\Delta^\top \widetilde{X}|
\}\ge \frac{\underline{f}r_0(\Delta^\top \widetilde{X})^2}{2(r_0 + |\Delta^\top \widetilde{X}|)}.
\]
Substituting this lower bound to \eqref{Eq:rho:diff} yields
\begin{align*}
&\cR(\omega^* + u, c^* + t) - \cR(\omega^*, c^*)\\
\ge&\mathbb{E}_X\left[ 
\frac{(b+h)\underline{f}r_0(\Delta^\top \widetilde{X})^2}{2(r_0 + |\Delta^\top \widetilde{X}|)}
\right]\ge \mathbb{E}_X\left[ 
\frac{(b+h)\underline{f}r_0(\Delta^\top \widetilde{X})^2}{2(r_0 + \|\Delta\|_2K)}
\right]\ge
\frac{ (b+h)\underline{f}r_0\kappa\|\Delta\|_2^2}{2(r_0 + \|\Delta\|_2K)},
\end{align*}
where the second and last inequalities leverage Assumptions~\ref{Assumption:structure:data}\ref{Assumption:structure:data:III} and \ref{Assumption:structure:data}\ref{Assumption:structure:data:IV}, respectively. 
\QED
\endproof
\begin{lemma}[Uniform Concentration Bound on Risk]\label{Lemma:uniform:bound}
Under Assumption~\ref{Assumption:structure:data}, there exists a constant $C_0$ depending on $b,h,K$ such that, with probability at least $1-\delta$, simultaneously for every $\Delta:=(u^\top,t)^\top\in\mathbb{R}^{D+1}$ with $\|u\|_0\le 2d$, it holds that
\[
\begin{multlined}
    \left|
\Big( 
\cR_n(\omega^* + u, c^* + t) - \cR_n(\omega^*, c^*)
\Big)
-
\Big( 
\cR(\omega^* + u, c^* + t) - \cR(\omega^*, c^*)
\Big)
\right|\\
\le C_0\|\Delta\|_2\cdot \sqrt{\frac{d\log(eD/d)+1+\log(2/\delta)}{n}}.
\end{multlined}
\]
\end{lemma}

\proof{Proof.}
The inequality for $\Delta=0$ holds trivially. We focus on the case where $\Delta\ne0$ instead.
Define the function $\ell(y,z)=b\max\{0,y-z\}+h\max\{0,z-y\}$, and
the function class
\[
\mathcal{H}=\left\{ 
(x,y)\mapsto 
\frac{\ell(y,x^\top(\omega^*+u)+c^*+t) - \ell(y, x^\top\omega^*+c^*)}{\|\Delta\|_2}:~\Delta\ne0, \Delta:=(u^\top,t)^\top, \|u\|_0\le 2d
\right\}.
\]
It suffices to provide concentration bound on 
\[
\sup_{h\in\mathcal{H}}~\Big| 
\mathbb{E}_{X,Y}[h(X,Y)] - \frac{1}{n}\sum_{i=1}^nh(x_i,y_i)
\Big|,
\]
which can be shown based on the covering number argument.

Each element in the function class $\mathcal{H}$ can be indexed with two terms: $r:=\|\Delta\|_2$ and $v:=\frac{\Delta}{\|\Delta\|_2}$. Hence, we reformulate
\[
\mathcal{H}=\left\{ 
\begin{aligned}
&(x,y)\mapsto h_{r,v}(x,y):=\frac{\ell(y,x^\top(\omega^*+rv_u)+c^*+rv_t) - \ell(y, x^\top\omega^*+c^*)}{r},\\
&r>0, v=(v_u^\top,v_t)^\top\in\mathbb{R}^{D+1}, \|v\|_2=1, \|v_u\|_0\le 2d
\end{aligned}
\right\}.
\]
Define the empirical Rademacher complexity
\[
\widehat{\mathfrak R}_n(\mathcal H)
:=
\mathbb E_{\sigma_i, i\in[n]}
\left[
\sup_{h\in\mathcal H}
\left|
\frac1n\sum_{i=1}^n\sigma_i h(x_i,y_i)
\right|
\Big|
(x_i,y_i), i\in[n]
\right],
\]
where $\sigma_i, i\in[n]$ are i.i.d. Rademacher random variables.
We provide upper bound on $\widehat{\mathfrak R}_n(\mathcal H)$ using the following procedure. 
\begin{itemize}
    \item 
For fixed $v\in\mathbb{R}^{D+1}$, denote the subset of $\mathcal{H}$ as
\[
\mathcal{H}_v = \Big\{ 
h_{r,v}:~r>0
\Big\}.
\]
Define the empirical $L_2$ pseudo-metric as
\[
\mathbf{d}_{n}(f,g) = \left\{ 
\frac{1}{n}\sum_{i=1}^n\Big( 
f(x_i,y_i) - g(x_i,y_i)
\Big)^2
\right\}^{1/2}.
\]
Let $r_1,\ldots,r_M$ with $0<r_1<\cdots<r_M$ index an $\delta$-packing of $\mathcal{H}_v$ under the metric $\mathbf{d}_{n}(\cdot)$. Therefore, 
\begin{align*}
(M-1)\delta^2&<\sum_{j=1}^{M-1}\mathbf{d}_{n}(h_{r_j,v}, h_{r_{j+1}, v})^2
=
\frac{1}{n}\sum_{i=1}^n\sum_{j=1}^{M-1}\Big( 
h_{r_j,v}(x_i,y_i) - h_{r_{j+1}, v}(x_i,y_i)
\Big)^2.
\end{align*}
According to the formulation of $h_{r,v}$ and by the convexity of $\ell(y,\cdot)$, it is easy to show that $h_{r,v}(x,y)$ is non-decreasing in $r$, and therefore
\begin{align*}
    (M-1)\delta^2&<\frac{1}{n}\sum_{i=1}^n\sum_{j=1}^{M-1}\Big( 
h_{r_j,v}(x_i,y_i) - h_{r_{j+1}, v}(x_i,y_i)
\Big)^2\\
&\le \frac{1}{n}\sum_{i=1}^n\left[ 
\sum_{j=1}^{M-1}\Big( 
h_{r_{j+1}, v}(x_i,y_i) - h_{r_j,v}(x_i,y_i)
\Big)
\right]^2\\
&=\frac{1}{n}\sum_{i=1}^n\Big( 
h_{r_{M}, v}(x_i,y_i) - h_{r_1,v}(x_i,y_i)
\Big)^2\le 4(b\lor h)^2K^2,
\end{align*}
where the last inequality uses the fact that for any $(x,y)$ and $r>0$,
\[
\left| 
h_{r,v}(x,y)
\right|\le (b\lor h)\cdot \frac{|x^\top(rv_u) + rv_t|}{r}\le (b\lor h)\|(x^\top,1)^\top\|_2\|v\|_2\le (b\lor h)K.
\]
The packing number $M\le 1 + 4(b\lor h)^2K^2/\delta^2$. Therefore, the covering number of $\mathcal{H}_v$ under $\mathbf{d}_n(\cdot)$ is bounded:
\[
\mathcal{N}\Big(
\delta, \mathcal{H}_v, \mathbf{d}_n
\Big)\le 1 + 4(b\lor h)^2K^2/\delta^2.
\]
\item
Next, we study the covering number of the set 
\[
\mathcal{V} = \Big\{ 
v=(v_u^\top,v_t)^\top\in\mathbb{R}^{D+1}:~\|v\|_2=1, \|v_u\|_0\le 2d
\Big\}.
\]
For fixed support $S\subseteq[D]$ with $|S|\le 2d$, the $\delta$-covering number of 
$\mathcal{V}_J=\Big\{ 
v=(v_u^\top,v_t)^\top\in\mathbb{R}^{D+1}:~\|v\|_2=1, \mathrm{supp}(v_u)=J
\Big\}$ under $\|\cdot\|_2$ is at most $(1 + 2/\delta)^{2d+1}$. By taking the union of all sparse supports, the covering number
\[
\mathcal{N}(\delta, \mathcal{V}, \|\cdot\|_2)\le \left( 
\sum_{k=0}^{\min\{2d,D\}}\binom{D}{k}
\right)(1 + 2/\delta)^{2d+1}.
\]
\item
For fixed $r>0$ and any two unit vectors $v=(v_u, v_t), v'=(v_u', v_t')$, it holds that 
\[
\Big| 
h_{r,v}(x,y) - h_{r, v'}(x,y)
\Big|\le (b\lor h)\cdot \frac{|x^\top(rv_u- rv_u') + rv_t - rv_t'|}{r}\le (b\lor h)K\|v - v'\|_2.
\]
For any $h_{r,v}\in \mathcal{H}$, we select $v'$ in the $\delta$-cover of $\mathcal{V}$ and $h_{r',v'}$ in the $\delta$-cover of $\mathcal{H}_{v'}$. It follows that
\begin{align*}
&\mathbf{d}_n(h_{r,v}, h_{r',v'})=\left[ 
\frac{1}{n}\sum_{i=1}^n\Big( 
h_{r,v}(x_i,y_i) -  h_{r',v'}(x_i,y_i)
\Big)^2
\right]^{1/2}\\
\le &\left[ 
\frac{2}{n}\sum_{i=1}^n\Big( 
h_{r,v}(x_i,y_i) -  h_{r,v'}(x_i,y_i)
\Big)^2 + \frac{2}{n}\sum_{i=1}^n\Big( 
h_{r',v'}(x_i,y_i) -  h_{r,v'}(x_i,y_i)
\Big)^2
\right]^{1/2}\\
\le&\left[ 
2(b\lor h)^2K^2\delta^2 + 2\delta^2
\right]^{1/2}=(2 + 2(b\lor h)^2K^2)^{1/2}\delta.
\end{align*}
Therefore, 
\[
\mathcal{N}\Big(
(2 + 2(b\lor h)^2K^2)^{1/2}\delta, \mathcal{H}, \mathbf{d}_n
\Big)\le \mathcal{N}\Big(
\delta, \mathcal{H}_v, \mathbf{d}_n
\Big)\cdot \mathcal{N}(\delta, \mathcal{V}, \|\cdot\|_2).
\]
It further implies that 
\begin{align*}
\log\mathcal{N}\Big(
\delta, \mathcal{H}, \mathbf{d}_n
\Big)
\le&\log\left[ 
1 + \frac{8(b\lor h)^2K^2(1 + (b\lor h)^2K^2)}{\delta^2}
\right] +  \log\left( 
\sum_{k=0}^{\min\{2d,D\}}\binom{D}{k}
\right)\\
&\qquad\qquad\qquad+ (2d+1)\log\left[ 
1 + \frac{2\sqrt{2}(1 + (b\lor h)^2K^2)^{1/2}}{\delta}
\right]\\
\le&C_1\cdot \left[
d\log\left(\frac{eD}{d}\right)
+1
+d\log\left(\frac{1}{\delta}\right)
\right],
\end{align*}
where $C_1$ is a constant depending on $b,h,K$.
As each element in $\mathcal{H}$ is uniformly bounded by $(b\lor h)K$, by Dudley's entropy integral bound~\citep{dudley1999uniform}, the Rademacher complexity
\begin{align*}
\mathfrak{R}_n(\mathcal{H})&
\le \frac{C_2}{\sqrt{n}}\int_0^{(b\lor h)K}\sqrt{\log\mathcal{N}\Big(
\delta, \mathcal{H}, \mathbf{d}_n
\Big)}\diff\delta\\
&\le C_3\sqrt{\frac{d\log\left(\frac{eD}{d}\right)
+1}{n}},
\end{align*}
where $C_2, C_3$ are some constants depending on $b,h,K$.
\end{itemize}
Based on the symmetrization and McDiarmid’s inequalities, it holds with probability at least $1-\delta$ that
\begin{align*}
\sup_{h\in\mathcal{H}}~\left| 
\frac{1}{n}\sum_{i=1}^nh(x_i,y_i) - \mathbb{E}[h(X,Y)]
\right|&\le 
C_4\sqrt{\frac{d\log\left(\frac{eD}{d}\right)
+1}{n}} + (b\lor h)K\sqrt{\frac{2\log(2/\delta)}{n}}\\
&\le C_0\sqrt{\frac{d\log(eD/d) + 1 + \log(2/\delta)}{n}},
\end{align*}
where $C_0, C_4$ are  some constants depending on $b,h,K$.
Re-arranging the inequality above yields the desired result.
\QED
\endproof

Finally, we provide the bound on the residual error
\[
a_n:=\left\|
\begin{pmatrix}
\widehat{\omega}_n - \omega^*\\
\widehat{c}_n - c^*
\end{pmatrix}
\right\|_2.
\]
\proof{Proof of Theorem~\ref{Theorem:error:bound}.}
As $\|\widehat{\omega}_n-\omega^*\|_0\le 2d$, by Lemma~\ref{Lemma:risk:lower}, we obtain the lower bound
\begin{equation}
\cR(\widehat{\omega}_n, \widehat{c}_n) - \cR(\omega^*, c^*)\ge 
\frac{(b+h)\underline{f}\kappa r_0}{2(r_0 + Ka_{n})}\cdot a_{n}^2.\label{Eq:lower:bound:r}
\end{equation}
Next, we derive its upper bound.
By the definition of $(\widehat{\omega}_n, \widehat{c}_n)$, it holds that
\[
\cR_n(\widehat{\omega}_n, \widehat{c}_n) + \frac{1}{2\gamma_n}\|\widehat{\omega}_n\|_2^2
\le 
\cR_n(\omega^*, c^*) + \frac{1}{2\gamma_n}\|\omega^*\|_2^2.
\]
Leveraging this relation and re-arranging the terms yields
\begin{align*}
&\cR(\widehat{\omega}_n, \widehat{c}_n) - \cR(\omega^*, c^*)\\
=&\Big[\cR_n(\widehat{\omega}_n, \widehat{c}_n) - \cR_n(\omega^*, c^*)\Big] 
-
\Big[ 
\Big( 
\cR_n(\widehat{\omega}_n, \widehat{c}_n) - \cR_n(\omega^*, c^*)
\Big)
-
\Big( 
\cR(\widehat{\omega}_n, \widehat{c}_n) - \cR(\omega^*, c^*)
\Big)
\Big]\\
\le&\frac{1}{2\gamma_n}\Big[ 
\|\omega^*\|_2^2 - \|\widehat{\omega}_n\|_2^2
\Big]-
\Big[ 
\Big( 
\cR_n(\widehat{\omega}_n, \widehat{c}_n) - \cR_n(\omega^*, c^*)
\Big)
-
\Big( 
\cR(\widehat{\omega}_n, \widehat{c}_n) - \cR(\omega^*, c^*)
\Big)
\Big].
\end{align*}
The first term above can be further bounded as follows:
\[
\frac{1}{2\gamma_n}\Big[ 
\|\omega^*\|_2^2 - \|\widehat{\omega}_n\|_2^2
\Big]\le
\frac{\langle \omega^*,  \omega^* - \widehat{\omega}_n\rangle}{\gamma_n}
\le \frac{\|\omega^*\|_2a_n}{\gamma_n}.
\]
By Lemma~\ref{Lemma:uniform:bound}, with probability at least $1-\delta$, the second term above can be bounded as
\[
C_0a_n\sqrt{\frac{d\log(eD/d)+1+\log(2/\delta)}{n}}.
\]
Combining these two relations implies the following relation holds with probability at least $1-\delta$:
\begin{equation}
\cR(\widehat{\omega}_n, \widehat{c}_n) - \cR(\omega^*, c^*)
\le a_n\cdot\left( 
\frac{\|\omega^*\|_2}{\gamma_n} + C_0\sqrt{\frac{d\log(eD/d)+1+\log(2/\delta)}{n}}
\right).\label{Eq:error:risk}
\end{equation}
Comparing this with \eqref{Eq:lower:bound:r} yields
\[
\frac{(b+h)\underline{f}\kappa r_0}{2(r_0 + Ka_{n})}\cdot a_{n}^2
\le a_n\cdot\left( 
\frac{\|\omega^*\|_2}{\gamma_n} + C_0\sqrt{\frac{d\log(eD/d)+1+\log(2/\delta)}{n}}
\right).
\]
Specify $\frac{1}{2\gamma_n}\lesssim \sqrt{\frac{d\log(eD/d)}{n}}$, then there exists constant $C_1$ depending on $C_0, b, h, K, \|\omega^*\|_2$ such that 
\begin{equation}
\frac{(b+h)\underline{f}\kappa r_0}{2(r_0 + Ka_{n})}\cdot a_{n}
\le C_1\sqrt{\frac{d\log(eD/d)+1+\log(2/\delta)}{n}}.\label{Eq:a:n}
\end{equation}
Assume $n$ is sufficiently large such that 
\[
C_1\sqrt{\frac{d\log(eD/d)+1+\log(2/\delta)}{n}}\le \frac{(b+h)\underline{f}\kappa r_0}{4K},
\]
then re-arranging \eqref{Eq:a:n} yields
\begin{align*}
a_n&\le \frac{\left[C_1\sqrt{\frac{d\log(eD/d)+1+\log(2/\delta)}{n}}\right]r_0}{
\frac{(b+h)\underline{f}\kappa r_0}{2} - \left[C_1\sqrt{\frac{d\log(eD/d)+1+\log(2/\delta)}{n}}\right]K
}\\
\le& \frac{\left[C_1\sqrt{\frac{d\log(eD/d)+1+\log(2/\delta)}{n}}\right]r_0}{\frac{(b+h)\underline{f}\kappa r_0}{4}}=\frac{4C_1}{(b+h)\underline{f}\kappa}\cdot \sqrt{\frac{d\log(eD/d)+1+\log(2/\delta)}{n}}
\end{align*}
holds with probability at least $1-\delta$.
\QED 
\endproof

\section{Proof of Corollary~\ref{Theorem:out:sample:bound}}
By definition,
\[
\min_{\|\omega\|_0\le d, c}~\mathcal{R}(\omega,c) = 
\min_{\|\omega\|_0\le d, c}~\left\{
(b+h)
\mathbb{E}_{(X,Y)}\left[ 
\rho_{\tau}(Y-(\omega^\top X + c))
\right]
\right\}.
\]
By the likelihood model, for each $x$, $(\omega^*)^\top x + c^*$ is the $\tau$-quantile of $Y\mid X=x$, and therefore the optimization problem above takes the optimal solution $(\omega^*, c^*)$.
Recall from the proof of Theorem~\ref{Theorem:error:bound} that for sufficiently large $n$, with probability at least $1-\delta$, the following two relations hold:
\begin{align*}
\cR(\widehat{\omega}_n, \widehat{c}_n) - \cR(\omega^*, c^*)&\le a_nC_1\cdot
\sqrt{\frac{d\log(eD/d)+1+\log(2/\delta)}{n}},\\
a_n&\le \frac{4C_1}{(b+h)\underline{f}\kappa}\cdot \sqrt{\frac{d\log(eD/d)+1+\log(2/\delta)}{n}}.
\end{align*}
Combining these two relations yields
\[
\cR(\widehat{\omega}_n, \widehat{c}_n) - \cR(\omega^*, c^*)\le \frac{4C_1^2}{(b+h)\underline{f}\kappa}\cdot \frac{d\log(eD/d)+1+\log(2/\delta)}{n}.
\]

\section{Proof of Corollary~\ref{Theorem:support}}\label{Sec:corollary:support}

Under the event described in Corollary~\ref{Theorem:support}, for each $j\in S^*$, it holds that
\begin{equation*}
\begin{aligned}
    |\widehat{\omega}_n[j]|\geq |\omega^*[j]| - |\widehat{\omega}_n[j]-\omega^*[j]|\geq |\omega^*_j| - \left\|
\begin{pmatrix}
\widehat{\omega}_n - \omega^*\\
\widehat{c}_n - c^*
\end{pmatrix}
\right\|_2>0.
\end{aligned}
\end{equation*}
Therefore, $S^*\subseteq\text{supp}(\widehat{\omega}_n)$. 
On the other hand, by the feasibility of $\widehat{\omega}_n$, it holds that $|\text{supp}(\widehat{\omega}_n)|\leq d=|S^*|$.
This fact, together with $S^*\subseteq\text{supp}(\widehat{\omega}_n)$ implies that $S^*=\text{supp}(\widehat{\omega}_n)$.

\end{document}